\documentclass{article}

\usepackage[nonatbib, final]{neurips_2024}

\usepackage[utf8]{inputenc} 
\usepackage[T1]{fontenc}    
\usepackage{hyperref}       
\usepackage{url}            
\usepackage{booktabs}       
\usepackage{amsfonts}       
\usepackage{nicefrac}       
\usepackage{microtype}      
\usepackage{xcolor}         
\usepackage{bm}
\usepackage{sidecap, caption}
\usepackage{hyperref}
\usepackage{url}
\usepackage{subfigure}
\usepackage{booktabs}       
\usepackage{nicefrac}       
\usepackage{microtype}      
\usepackage{color}         
\usepackage{threeparttable}
\usepackage{amsmath}
\usepackage{cases}
\usepackage{epstopdf}
\usepackage{multirow}
\usepackage{fancybox}
\usepackage{tikz}
\usepackage{environ}
\usepackage{soul}
\usepackage{amsthm}
\usepackage{bm}
\usepackage{bbm}
\usepackage{latexsym}
\usepackage{amsfonts,amssymb}
\usepackage{stfloats}
\usepackage{caption}
\usepackage{algorithm}
\usepackage{algorithmic}

\usepackage{wrapfig}
\newtheorem{theorem}{Theorem}[section]

\newtheorem{lemma}[theorem]{Lemma}

\newtheorem{definition}[theorem]{Definition}

\newcommand{\eg}{{\em e.g.,~}}     
\newcommand{\ie}{{\em i.e.,~}}      

\hypersetup{
    colorlinks=true,                
    linkcolor= red,                
    citecolor=[rgb]{0,0.5,0.73},
}

\title{Revisiting Self-Supervised Heterogeneous Graph Learning from Spectral Clustering Perspective}

\author{\textbf{Yujie Mo$^{1,2}$} \quad  \textbf{Zhihe Lu$^2$} \quad \textbf{Runpeng Yu$^2$} \quad \textbf{Xiaofeng Zhu$^{1*}$} \quad
\textbf{Xinchao Wang$^{2}$}\thanks{Corresponding authors}\\
$^1$School of Computer Science and Engineering, \\
University of Electronic Science and Technology of China\\
$^2$National University of Singapore\\
}

\begin{document}

\maketitle

\begin{abstract}
Self-supervised heterogeneous graph learning (SHGL) has shown promising potential in diverse scenarios.
However,  while existing SHGL methods share a similar essential with clustering approaches, they encounter two significant limitations: (i) noise in graph structures is often introduced during the message-passing process to weaken node representations, and (ii) cluster-level information may be inadequately captured and leveraged, diminishing the performance in downstream tasks. 
In this paper, we address these limitations by theoretically revisiting SHGL from the spectral clustering perspective and introducing a novel framework enhanced by rank and dual consistency constraints.
Specifically, our framework incorporates a rank-constrained spectral clustering method that refines the affinity matrix to exclude noise effectively. 
Additionally, we integrate node-level and cluster-level consistency constraints that concurrently capture invariant and clustering information to facilitate learning in downstream tasks. 
We theoretically demonstrate that the learned representations are divided into distinct partitions based on the number of classes and exhibit enhanced generalization ability across tasks. 
Experimental results affirm the superiority of our method, showcasing remarkable improvements in several downstream tasks compared to existing methods.

\end{abstract}


\section{Introduction}

Self-supervised heterogeneous graph learning (SHGL) aims to effectively process diverse types of nodes and edges in the heterogeneous graph, producing low-dimensional representations without the need for human annotations \cite{CKD00010YCLF022, ZhaoWSHSY21, liu2022revisiting}. 
Thanks to its remarkable capabilities, SHGL has attracted significant interest and has been utilized in a broad array of applications, including recommendation systems \cite{shi2018heterogeneous, hu2018leveraging}, social network analysis \cite{shi2016survey, he2022analyzing}, and molecule design \cite{yu2022molecular, wu2023molformer}.



Existing SHGL methods can be broadly classified into two groups, \ie meta-path-based methods and adaptive-graph-based methods. 
Meta-path-based methods typically utilize pre-defined meta-paths to explore relationships among nodes that may share the same label in the heterogeneous graph \cite{jing2021hdmi, zhu2022structure}. 
However, building meta-paths requires extensive prior knowledge and incurs additional computation costs \cite{zhang2022simple}.
To address these drawbacks, adaptive-graph-based methods dynamically assign significant weights to node pairs likely to share the same label, using the adaptive graph structure rather than traditional meta-paths \cite{mo2024selfsupervised}.
Both groups of SHGL methods facilitate message-passing among nodes within the same class, either through meta-path-based graphs or adaptive graph structures. 
As a result, this process minimizes intra-class differences and promotes a clustered pattern in the learned representations, aligning these methods closely with conventional clustering techniques.

\begin{figure*}[!t]
\centering
		\subfigure{{\includegraphics[scale=0.79]{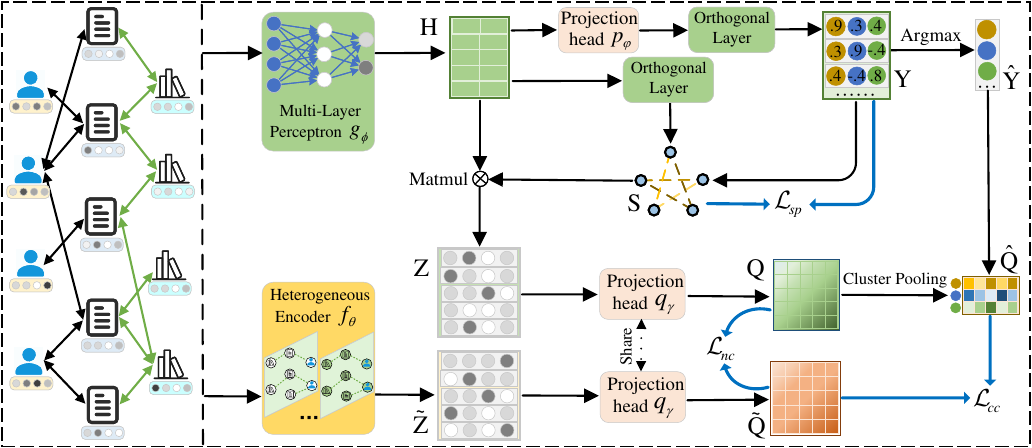
  }}}
		\caption{
 The flowchart of SCHOOL, which first employs the Multi-Layer Perception $g_{\phi}$ to derive semantic representations $\mathbf{H}$, followed by obtaining orthogonal cluster assignment matrix $\mathbf{Y}$ and orthogonal  ${\mathbf{H}}$. Subsequently, SCHOOL filters noisy connections by deriving the rank-constrained affinity matrix $\mathbf{S}$, which is further used to multiply with $\mathbf{H}$ and then obtain node representations $\mathbf{Z}$. Meanwhile, SCHOOL employs a heterogeneous encoder $f_{\theta}$ to aggregate information across node types, yielding heterogeneous representations $\tilde{\mathbf{Z}}$. 
 Finally, SCHOOL incorporates spectral loss $\mathcal{L}_{sp}$ to optimize $\mathbf{Y}$ to fit eigenvectors of the Laplacian matrix of $\mathbf{S}$. Moreover, SCHOOL designs node-level (\ie $\mathcal{L}_{nc}$) and cluster-level (\ie $\mathcal{L}_{cc}$) consistency constraints on projected representations (\ie $\mathbf{Q}$ and $\tilde{\mathbf{Q}}$) and cluster representations $\hat{\mathbf{Q}}$ to capture the invariant and clustering information, respectively.
         }
		\label{frame}
\end{figure*}

Despite the effectiveness of previous SHGL methods, they encounter two limitations. 
First, previous methods conduct message-passing relying on meta-path-based graphs and adaptive graph structures, which inevitably include noise, \ie connections among nodes from different classes \cite{jing2021hdmi, wang2023heterogeneous}. Consequently, such noise compromises the identifiability of node representations after the message-passing process. 
Second, while previous methods exhibit clustering characteristics, they typically emphasize the node-level consistency only, neglecting to capture and leverage the cluster-level information effectively \cite{WangLHS21, mo2024selfsupervised}. 
This may not fully exploit the potential benefits of clustering for representation learning, thereby diminishing the performance of downstream tasks.


Based on the above analysis, it is feasible to analyze previous SHGL methods from a clustering perspective thanks to their close connection to clustering techniques and further optimize the graph structures to mitigate noisy connections as well as harness the cluster-level information to enhance previous SHGL. 
To achieve this, there are three key challenges, 
\ie (i) How to formally understand previous SHGL methods from the clustering perspective? 
(ii) With insights from the clustering analysis, how to learn an adaptive graph structure that effectively captures intra-class connections while filtering out inter-class noise?
(iii) How to enable the effective incorporation of cluster-level information in the heterogeneous graph to boost the performance of downstream tasks?




In this paper, we address the outlined challenges by first theoretically revisiting previous SHGL methods from a clustering perspective and then introducing a novel framework, termed \underline{S}pectral \underline{C}lustering-inspired \underline{H}eter\underline{O}gene\underline{O}us graph \underline{L}earning (SCHOOL for short), that incorporates rank-constrained spectral clustering and dual consistency constraints, as depicted in Figure \ref{frame}.
Specifically, we start by proving that existing SHGL can be reformulated as spectral clustering with an additional regularization term under the assumption of orthogonality, thus addressing \textbf{challenge (i)} and laying the foundational theory for our approach. 
Next, to tackle \textbf{challenge (ii)}, we propose an efficient spectral clustering technique that includes a rank constraint on the affinity matrix, aiming to effectively mitigate noisy connections among different classes.
Furthermore, to resolve \textbf{challenge (iii)}, we design dual consistency constraints at both node and cluster levels to capture invariant and clustering information, respectively, which reduces the intra-cluster differences and enhances the performance of downstream tasks. 
Finally, theoretical analysis indicates that the learned representations are divided into distinct partitions corresponding to the number of classes, and are demonstrated superior generalization ability compared to those derived from previous SHGL methods.

 Compared to previous SHGL works, our contributions can be summarized as follows:
\begin{itemize}
    \item To the best of our knowledge, we make the first attempt to theoretically revisit previous SHGL methods from the spectral clustering perspective in a unified manner.

    \item We adaptively learn a rank-constrained affinity matrix to mitigate noisy connections inherent in previous SHGL methods. Moreover, we introduce dual consistency constraints to capture both invariant and clustering information to enhance the effectiveness of our method.


    \item We theoretically demonstrate that the proposed method divides the learned representations into distinct partitions based on the number of classes, instead of dimensions in previous SHGL methods. Furthermore, the representations obtained by this method exhibit enhanced generalization ability compared to those derived from previous SHGL methods.

    \item We experimentally manifest the effectiveness of the proposed method across a variety of downstream tasks, using both heterogeneous and homogeneous graph datasets, compared to numerous state-of-the-art methods.
\end{itemize}

\section{Method}
\label{method}
$\textbf{Notations.}$ Let $\mathbf{G}=(\mathcal{V}, \mathcal{E}, \mathbf{X}, \mathcal{T}, \mathcal{R})$ represent a heterogeneous graph, where $\mathcal{V}$ and $\mathcal{E}$ indicate set of nodes and set of edges, respectively. $\mathbf{X}= \left\{\mathbf{x}_{i}\right\}_{i=1}^{n}$ denotes the matrix of node features, where $n$ indicates the number of nodes. Moreover, $\mathcal{T}$ and $\mathcal{R}$ indicate set of node types and set of edge types, respectively. Given the heterogeneous graph $\mathbf{G}$, most existing SHGL methods utilize meta-paths or adaptive graph structures to explore connections among nodes within the same class, thus exhibiting the characteristics of clustering and obtaining discriminative representations. 
To gain a deeper insight of previous SHGL methods, we first propose to revisit them from a clustering perspective as follows.

\subsection{Revisiting Previous SHGL Methods from Spectral Clustering}


As mentioned above, previous SHGL methods tend to conduct clustering implicitly, relying on meta-path-based graphs or adaptive graph structures. For example, given an academic heterogeneous graph with several node types (\ie paper, author, and subject), for the meta-path-based methods, if two papers belong to the same class, there may exist a meta-path “paper-subject-paper” to connect them (\ie two papers are grouped into the same subject). Similarly, for the adaptive-graph-based methods, when two papers belong to the same class, the adaptive graph structures likely assign large weights to connect them. Therefore, representations of nodes within the same class will be close to each other after the message-passing process, thus implicitly presenting a clustered pattern. 


Based on the above observation, actually, we can further theoretically understand previous SHGL methods from the clustering perspective. To do this, we first give the following definition.
\begin{definition} (Spectral Clustering)
    Given the Laplacian matrix $\mathbf{L}$, the optimization problem of the spectral clustering can be described as follows:
    \begin{equation}
        \min _{\mathbf{H}} \operatorname{Tr}\left(\mathbf{H}^{T} \mathbf{L} \mathbf{H}\right) \operatorname{\text { s.t., }} \mathbf{H}^{T} \mathbf{H}=\mathbf{I},
    \end{equation}
    where $\mathbf{L} = \mathbf{D}-\mathbf{W}$,  $\mathbf{W} \in \mathbb{R}^{n \times n}$ is a  data similarity matrix, $\mathbf{D}$ is  a diagonal matrix whose entries are column sums of $\mathbf{W}$, $\mathbf{H} \in \mathbb{R}^{n \times d}$ is  a representations matrix, $\operatorname{Tr}(\cdot)$ indicates the matrix trace, and $\mathbf{I}$ indicates the identity matrix.
\label{def21}
\end{definition}
According to Definition \ref{def21}, for both meta-path-based methods \cite{CPIM,WangLHS21, wang2023heterogeneous} and adaptive-graph-based SHGL methods \cite{mo2024selfsupervised}, we then have  Theorem \ref{thm1}, whose proof can be found in Appendix \ref{proof22}.
\begin{theorem}
Assume the learned representations $\mathbf{H}$ are orthogonal, optimizing previous meta-path-based  and adaptive-graph-based SHGL methods is equivalent to performing spectral clustering with additional regularization, i.e.,
\begin{equation}
    \min _{\mathbf{H}}\mathcal{L}_{SHGL}\cong \min _{\mathbf{H}} \operatorname{Tr}(\mathbf{H}^{T} \hat{\mathbf{L}} \mathbf{H}) + R(\mathbf{H}) \text { s.t., } \mathbf{H}^{T} \mathbf{H}=\mathbf{I} \text {, }
\end{equation}
where $R(\cdot)$ indicates the regularization term, $\hat{\mathbf{L}}$ indicates the Laplacian matrix of the meta-path-based graph or the adaptive graph structure.
\label{thm1}
\end{theorem}

Theorem \ref{thm1} reveals the connection between previous SHGL and the spectral clustering as well as indicates that previous SHGL heavily relies on the Laplacian matrix of meta-path-based graph or adaptive graph structure. 
Moreover, based on Theorem \ref{thm1}, we can further bridge previous SHGL and the graph-cut algorithm \cite{wang2003image} as follows, whose proof can be found in Appendix \ref{proof23}.

\begin{theorem}
Under the same assumption in Theorem \ref{thm1}, optimizing previous meta-path-based  and adaptive-graph-based SHGL methods is approximate to performing the $\operatorname{RatioCut}\left(V_{1}, \ldots, V_{d}\right)$  algorithm that divides the learned representations into $d$ partitions $\{V_{1}, \ldots, V_{d}\}$, i.e.,
\begin{equation}
        \min _{\mathbf{H}}\mathcal{L}_{SHGL}\cong \min _{\mathbf{H}} \operatorname{RatioCut}\left(V_{1}, \ldots, V_{d}\right),
\end{equation}
where $d$ indicates the dimension of representations $\mathbf{H}$.
\label{thm2}
\end{theorem}

Theorem \ref{thm2} further indicates that previous SHGL methods divide the learned representations into $d$ partitions, where $d$ is generally much larger than the number of classes. Therefore, Theorem \ref{thm2} connects the traditional graph-cut algorithm with existing SHGL methods, which requires custom analysis. As a result, we theoretically revisit previous SHGL methods from spectral clustering as well as graph-cut perspectives and build the connections between them, thus solving the challenge (i).

\subsection{Rank-Constrained Spectral Clustering}
Based on the connections between previous SHGL methods and the spectral clustering as well as the graph-cut algorithm, we have the observations as follows. First, according to Theorem \ref{thm1}, previous SHGL methods conduct  spectral clustering based on the Laplacian matrix of meta-path-based graph or adaptive graph structure, which may not guarantee optimality and could potentially contain noisy connections, thus affecting the spectral clustering. Second, according to Theorem \ref{thm2}, previous SHGL methods conduct the graph-cut to divide the learned representations into $d$ partitions, which are generally not equal to the number of classes $c$. As a result, optimizing previous SHGL methods becomes a hard or even error problem, and the learned representations can not be clustered well.
Therefore, it is intuitive to mitigate noise in the adaptive graph structure as well as divide the learned representations into exactly $c$ partitions to improve existing SHGL methods. 

Specifically, in this paper, we propose to learn an adaptive affinity matrix with the rank constraint to mitigate noisy connections as much as possible. To do this, we first employ the  Multi-Layer Perceptron (MLP) as the encoder $g_{\phi} \in \mathbb{R}^{f \times d_1}$  to obtain the semantic representations $\mathbf{H}$  by: 
\begin{equation}
    \mathbf{H} =\sigma (g_\phi(\mathbf{X})),
    \label{eq1}
\end{equation}
where $f$ and $d_1$ are the dimensions of node features and representations, respectively, and $\sigma$ is the activation function. After that, we propose to learn an adaptive affinity matrix $\mathbf{S} \in \mathbb{R}^{n \times n}$ based on the semantic representations. Intuitively, in a uncorrelated representations subspace,  a smaller distance $\|\mathbf{h}_i-\mathbf{h}_j\|_2^2$ between semantic representations should be assigned a larger probability $\mathbf{s}_{ij}$. Therefore, it is a natural approach to learn the affinity matrix $\mathbf{S}$ by:
\begin{equation}
    \begin{array}{l}
\min_\mathbf{S} \sum_{i, j=1}^{n}(\|\mathbf{h}_{i}-\mathbf{h}_{j}\|_{2}^{2} \mathbf{s}_{i j}+\alpha\mathbf{s}_{i j}^{2}) ~~
\text {s.t.,} \forall i, \mathbf{s}_{i}^{T} \mathbf{1}=1,0 \leq \mathbf{s}_{i} \leq 1,
\end{array}
\label{eq5}
\end{equation}
where $\alpha$ is a non-negative parameter. 
In Eq. (\ref{eq5}), the first term encourages the affinity matrix to assign large weights to node pairs with small distances. Moreover, the second term avoids the trivial solution that only the nearest node can be the neighbor of $v_i$ with probability 1.
However, similar to previous SHGL methods, usually the affinity matrix learned by Eq. (\ref{eq5}) can not reach the ideal case (\ie having no noisy connections among different classes and containing exactly $c$ connected components). As a result, the noisy connections in the affinity matrix may induce a negative interference during the message-passing process. To solve this issue, we first introduce  the following lemma in \cite{mohar1991laplacian}.
\begin{lemma}
The multiplicity $c$ of the eigenvalue 0 of the
Laplacian matrix $\mathbf{L}_\mathbf{S}$ is equal to the number of connected components in the affinity matrix $\mathbf{S}$.
\label{lem1}
\end{lemma}
Lemma \ref{lem1} indicates that if the rank 
of $\mathbf{L}_{\mathbf{S}}$ equals to $n-c$, then the affinity matrix $\mathbf{S}$ contains exactly $c$ connected components to achieve the ideal scenario, where $\mathbf{L}_{\mathbf{S}} = \mathbf{D} - \frac{\mathbf{S}+\mathbf{S}^T}{2}$, and $\mathbf{D}$ is the degree matrix of $\frac{\mathbf{S}+\mathbf{S}^T}{2}$. Based on Lemma \ref{lem1}, we can solve the above issue by adding the rank constraint on the affinity matrix, \ie enforcing the smallest $c$ eigenvalues of $\mathbf{L}_{\mathbf{S}}$ to be 0: 
\begin{equation}
    \operatorname{rank}(\mathbf{L}_{\mathbf{S}})=n-c \Rightarrow   \min \sum_{i=1}^{c} \tau _{i}\left(\mathbf{L}_\mathbf{S}\right), 
    \label{eq6}
\end{equation}
where $\tau _{i}(\mathbf{L}_\mathbf{S})$ is the $i$-th smallest eigenvalue of $\mathbf{L}_\mathbf{S}$ and  $\tau _{i}(\mathbf{L}_\mathbf{S}) \ge 0$ because $\mathbf{L}_\mathbf{S}$ is positive semi-definite. 
Moreover, according to Ky Fan’s Theorem \cite{fan1949theorem}, the constraint in Eq. (\ref{eq6}) can be rewritten in the spectral clustering form as follows (derivation listed in Appendix \ref{derivationeq7}).
\begin{equation}
       \sum_{i=1}^{c} \tau _{i}\left(\mathbf{L}_\mathbf{S}\right) 
    = \min _{\mathbf{F}^{T} \mathbf{F}=\mathbf{I}} \operatorname{Tr}(\mathbf{F}^{T} \mathbf{L}_\mathbf{S} \mathbf{F})
    = \min _{\mathbf{F}^{T} \mathbf{F}=\mathbf{I}} \frac{1}{2} \sum_{i, j} \mathbf{s}_{i j}\left\|\mathbf{f}_{i}-\mathbf{f}_{j}\right\|_{2}^{2},
    \label{eq7}
\end{equation}
where $\mathbf{F} \in \mathbb{R}^{n \times c}$ is formed by part eigenvectors of $\mathbf{L}_\mathbf{S}$ after the eigendecomposition. Therefore, based on Eq. (\ref{eq7}), we can  add the rank constraint on the affinity matrix and reformulate Eq. (\ref{eq5}) as:
\begin{equation}
        \begin{array}{l}\min_{\mathbf{S},\mathbf{F}} \sum_{i, j=1}^{n}(\|\mathbf{h}_{i}-\mathbf{h}_{j}\|_{2}^{2} \mathbf{s}_{i j}+\alpha\mathbf{s}_{i j}^{2}+\beta\left\|\mathbf{f}_{i}-\mathbf{f}_{j}\right\|_{2}^{2}\mathbf{s}_{i j}) \\
\text { s.t., } \forall i, \mathbf{s}_{i}^{T} \mathbf{1}=1,0 \leq \mathbf{s}_{i} \leq 1, \mathbf{F}^T\mathbf{F}=\mathbf{I},
\end{array}
\label{eq8}
\end{equation}
where $\beta$ is a non-negative parameter. Eq. (\ref{eq8}) can be solved by applying the alternating optimization approach. Specifically, when $\mathbf{S}$ is fixed, Eq. (\ref{eq8}) becomes $\min_{\mathbf{F}^T\mathbf{F}=\mathbf{I}}\sum_{i, j=1}^{n}\left\|\mathbf{f}_{i}-\mathbf{f}_{j}\right\|_{2}^{2}\mathbf{s}_{i j}=2\min_{\mathbf{F}^{T} \mathbf{F}=\mathbf{I}} \operatorname{Tr}(\mathbf{F}^{T} \mathbf{L}_\mathbf{S} \mathbf{F})$, whose optimal solution $\mathbf{F}$ is the eigenvectors of $\mathbf{L}_\mathbf{S}$ corresponding to the $c$ smallest eigenvalues.  When $\mathbf{F}$ is fixed, denote $\mathbf{d}_{ij} = \|\mathbf{h}_{i}-\mathbf{h}_{j}\|_{2}^{2} + \beta\left\|\mathbf{f}_{i}-\mathbf{f}_{j}\right\|_{2}^{2}$, Eq. (\ref{eq8})  can be written in the vector form, \ie
\begin{equation}
    \min _{\mathbf{s}_{i}^{T} \mathbf{1}=1,0 \leq \mathbf{s}_{i} \leq 1,} \|\mathbf{s}_{i}+\frac{1}{2 {\alpha}} \mathbf{d}_{i}\|_{2}^{2}.
    \label{eq9}
\end{equation}
According to the Lagrangian function of Eq. (\ref{eq9}) and  the KKT condition \cite{boyd2006convex}, we can obtain the closed-form solution of the element $\mathbf{s}_{ij}$ in the affinity matrix, \ie
\begin{equation}
    \mathbf{s}_{ij} = (-\frac{1}{2 {\alpha}} \mathbf{d}_{ij} + \lambda )_+,
    \label{eq10}
\end{equation}
where $\lambda$ is a non-negative parameter, and $(\cdot)_+$ indicates $\max \{\cdot, 0\}$. In practice, a sparse affinity matrix usually obtains better performance and reduces computation costs. Therefore, we only calculate $\mathbf{s}_{ij}$ between node $v_i$ and its $k$ nearest neighbors.
Then parameters $\alpha$ and $\lambda$ can be further determined, details listed in Appendix \ref{closedform}.

However, when fixing $\mathbf{S}$ and optimizing $\mathbf{F}$ in Eq. (\ref{eq8}), computation costs of obtaining eigenvectors $\mathbf{F}$ remain prohibitive due to the cubic time complexity of the eigendecomposition.
To address this issue, we propose to replace the eigendecomposition with a projection head and orthogonal layer \cite{shaham2018spectralnet}. Specifically, we first employ the projection head $p_\varphi  \in \mathbb{R}^{d_1 \times c}$ to map semantic representations to the cluster assignment space $\mathbf{P} \in \mathbb{R}^{n \times c}$, \ie
\begin{equation}
    \mathbf{P} =\sigma (p_\varphi(\mathbf{H})),
    \label{eq11}
\end{equation}
where $\sigma$ is  the activation function. After that, we employ an orthogonal layer to derive the orthogonal cluster assignment matrix $\mathbf{Y} \in \mathbb{R}^{n \times c}$ (orthogonal derivation listed in Appendix \ref{orthgo}), \ie
\begin{equation}
    \mathbf{Y}=\sqrt{n} \mathbf{P}\left(\mathbf{R}^{-1}\right),
    \label{eq12}
\end{equation}
where $\mathbf{R}$ is an upper triangular matrix obtained from the QR decomposition \cite{gander1980algorithms} (\ie $\mathbf{P} = \mathbf{E}\mathbf{R}$ and $\mathbf{E}^T\mathbf{E}=\mathbf{I}$) on the full rank $\mathbf{P}$. Similarly, we can also implement the orthogonal layer to achieve the uncorrelation (\ie $\mathbf{H}^T\mathbf{H}=\mathbf{I}$) on the representations subspace. 

As a result, the projection head and the orthogonal layer act as a linear transformation to achieve the orthogonality constraint on $\mathbf{Y}$. To replace the eigendecomposition, the cluster assignment matrix $\mathbf{Y}$ should further fit the eigenvectors $\mathbf{F}$ in Eq. (\ref{eq8}). 
To do this, we design a spectral loss $\mathcal{L}_{sp}$ to optimize the parameters in $p_\varphi$ to simulate the spectral clustering of the third term in Eq. (\ref{eq8}), \ie
\begin{equation}
    \mathcal{L}_{sp}=\frac{1}{n^{2}} \textstyle\sum_{i, j=1}^{n} \mathbf{s}_{i j}\left\|{\mathbf{y}}_{i}-{\mathbf{y}}_{j}\right\|^{2}_2 - \gamma H(\mathbf{Y}),
    \label{eq13}
\end{equation}
where $\gamma$ is a non-negative parameter, $H(\mathbf{Y}) = - {\textstyle \sum_{i=1}^{c}} P(\mathbf{y}^i)\text{log}P(\mathbf{y}^i)$ is the entropy of cluster assignment probabilities  $P(\mathbf{y}^i) =\frac{1}{n} {\textstyle \sum_{j=1}^{n}} \mathbf{y}_j^i$, and $\mathbf{y}_j^i$ indicates the $i$-th column and $j$-th row of $\mathbf{y}$.
According to Eq. (\ref{eq7}), the first term in Eq. (\ref{eq13}) simulates the spectral clustering to enforce the learned cluster assignment matrix $\mathbf{Y}$ to approximate eigenvectors $\mathbf{F}$, and the second term is a widely used regularization term \cite{bai2020sparse, huang2020deep} to avoid the trivial solution that most nodes are assigned to the same cluster. Therefore, when $\mathbf{S}$ is fixed, we optimizing $\mathbf{Y}$ to approach the optimal $\mathbf{F}$ by achieving orthogonality with Eq. (\ref{eq12}), and fitting eigenvectors $\mathbf{F}$ with Eq. (\ref{eq13}). As a result, this reduces the cubic time complexity of eigendecomposition to $\mathcal{O}(nd^2+ nc^2+ nkd + nkc + c^3 + d^3)$, where $d^2, c^2<n$, thus is linear to the sample size, details are shown in Appendix \ref{complexity}.
Finally, the proposed method still optimizes the affinity matrix $\mathbf{S}$ and eigenvectors $\mathbf{F}$ in Eq. (\ref{eq8}) in an alternating approach, \ie fix $\mathbf{F}$ and then obtain the closed-form solution of $\mathbf{S}$, and fix $\mathbf{S}$ and then optimize parameters (\ie $g_\phi$ and $p_\varphi$) to update $\mathbf{Y}$ to approach the optimal $\mathbf{F}$.

Therefore, the proposed method obtains the affinity matrix with exactly $c$ connected components to mitigate noisy connections in an effective and efficient way. Then we can obtain the node representations $\mathbf{Z} = \mathbf{S}\mathbf{H}$, which is expected to conduct message-passing among nodes within the same class. Moreover, we can bridge the connections between the proposed method and the spectral clustering as well as the graph-cut algorithm as follows, whose proof can be found in Appendix \ref{proof25}.
\begin{theorem}
Optimizing the spectral loss $\mathcal{L}_{sp}$ leads to performing the spectral clustering  based on the affinity matrix $\mathbf{S}$ with $c$ connected components and conducting RatioCut ($V_1,\ldots, V_{c}$) algorithm to divide the learned representations into $c$ partitions, i.e.,
\begin{equation}
\min\mathcal{L}_{sp}\Rightarrow\min\operatorname{Tr} (\mathbf{Y}^T\mathbf{L}_\mathbf{S}\mathbf{Y})\Rightarrow\min \operatorname{RatioCut}(V_{1},\ldots,V_{c}).
\end{equation}
\label{thm3}
\end{theorem}
Theorem \ref{thm3} indicates that the proposed method conducts the spectral clustering as previous SHGL methods, but is performed on an affinity matrix with exactly $c$ connected components (verified in Section \ref{visualization}), thus mitigating noisy connections from different classes and solving the challenge (ii). Moreover, the proposed method divides the learned representations into $c$ partitions, which is a better optimization goal than previous SHGL methods to obtain discriminative representations.

\subsection{Dual Consistency Constraints}
The message-passing among nodes within the same class reduces intra-class differences and enhances node representations $\mathbf{Z}$. Meanwhile, the message-passing among nodes from different types also contributes to obtaining task-related contents and benefits downstream tasks \cite{zhang2022simple}. To do this, we propose to aggregate the information of nodes from different types in the heterogeneous graph with a heterogeneous encoder $f_\theta  \in \mathbb{R}^{f \times d_1}$.

Specifically, for the node $v_i$, we concatenate the information of itself and its relevant one-hop neighbors (\ie nodes of other types) based on edge types in $\mathcal{R}$, and then derive the heterogeneous representations $\tilde{\mathbf{Z}}$ by:
\begin{equation}
    \tilde{\mathbf{z}}_i = \frac{1}{|\mathcal{R}|} \sum_{r \in \mathcal{R}}^{} \sigma(f_\theta(\mathbf{x}_i)|| \sum_{v_j \in \mathcal{N}_{i,r}}^{}f_\theta(\mathbf{x}_j) ),
\end{equation}
where  $\sigma$ is the activation function, $|\mathcal{R}|$ indicates the number of edge types, $\mathcal{N}_{i,r}$ indicates the set of one-hop neighbors of node $v_i$ based on the edge type $r \in \mathcal{R}$, $f_\theta(\cdot)$ indicates the linear transformation, and $\cdot || \cdot$ indicates the concatenation operation. Therefore, the
heterogeneous representations $\tilde{\mathbf{Z}}$ aggregate the information of nodes from different types to introduce more task-related contents.


Given node representations $\mathbf{Z}$ and heterogeneous representations $\tilde{\mathbf{Z}}$, most previous SHGL methods  utilize the node-level consistency constraint (\eg Info-NCE loss \cite{oord2018representation}) to capture the invariant information between them and enhance the effectiveness \cite{WangLHS21, mo2024selfsupervised}. In addition, according to Theorem \ref{thm1}, previous SHGL methods actually perform spectral clustering to learn node representations. However, previous SHGL methods fail to utilize the cluster-level information outputted by the spectral clustering, thus weakening the downstream task performance. To solve this issue, we design dual consistency constraints to capture the invariant information as well as the clustering information between $\mathbf{Z}$ and $\tilde{\mathbf{Z}}$.

Specifically, we first employ a projection head $q_\gamma  \in \mathbb{R}^{d_1 \times d_2}$ to map both  $\mathbf{Z}$ and $\tilde{\mathbf{Z}}$ into the same latent space, \ie  $\mathbf{Q} = q_\gamma(\mathbf{Z})$ and $\tilde{\mathbf{Q}} = q_\gamma(\tilde{\mathbf{Z}})$, where $d_2$ is the projected dimension. Then we follow previous works \cite{WangLHS21, wang2023heterogeneous} to design a node-level consistency constraint to capture the invariant information between $\mathbf{Q}$ and $\tilde{\mathbf{Q}}$, \ie
\begin{equation}
    \mathcal{L}_{nc} = \| \mathbf{Q}-\tilde{\mathbf{Q}}  \|^2_F + \eta \operatorname{log}\textstyle \sum_{i, j=1}^{d}e^{\mathbf{c}_{ij}},
    \label{eq15}
\end{equation}
where $\mathbf{C} = \mathbf{Q}^T\mathbf{Q} + \tilde{\mathbf{Q}}^T\tilde{\mathbf{Q}}$, and $\eta$ is a non-negative parameter. Similar to previous works,  the first term in Eq. (\ref{eq15}) enforces representations in $\tilde{\mathbf{Q}}$ agree with the corresponding representations in  $\mathbf{Q}$, thus capturing the invariant information between them. The second term enables different dimensions of $\mathbf{Q}$ and $\tilde{\mathbf{Q}}$ to uniformly distribute over the latent space to avoid collapse. 

In addition to the node-level consistency constraint, we further design the cluster-level consistency constraint to capture the clustering information from the cluster assignment matrix $\mathbf{Y}$. To do this, we first obtain the cluster indicator matrix $\hat{\mathbf{Y}}$ based on the cluster assignment matrix $\mathbf{Y}$, \ie $\hat{\mathbf{Y}} = \operatorname{argmax(\mathbf{Y})}$. After that, we conduct average pooling on node representations that possess the same cluster indicator to obtain  the $j$-th cluster representation $\hat{\mathbf{q}}_j$, \ie
\begin{equation}
    \hat{\mathbf{q}}_j = \frac{1}{|V_j|} \textstyle\sum_{v_i \in V_j}^{} \mathbf{q}_i,
    \label{eq16}
\end{equation}
where $V_j$ indicates the  set of nodes whose cluster indicators equal to $j$, and $|V_j|$ indicates the number of nodes in $V_j$. Then we design a cluster-level consistency constraint on cluster representations $\hat{\mathbf{Q}}$ and projected representations $\tilde{\mathbf{Q}}$ to capture the clustering information, \ie 
\begin{equation}
    \mathcal{L}_{cc} =\textstyle \sum_{i=1}^{n}\| \tilde{\mathbf{q}}_i -\hat{\mathbf{q}}_{\mathbf{y}_i} \|_2^2,
    \label{eq17}
\end{equation}
where $\hat{\mathbf{q}}_{\mathbf{y}_i}$ indicates the cluster representation whose label equals to $\mathbf{y}_i$.
Eq. (\ref{eq17}) enables the projected representation $\tilde{\mathbf{q}}_i$ and the cluster representation $\hat{\mathbf{q}}_{\mathbf{y}_i}$ to align each other.  As a result, representations capture the clustering information based on cluster indicators  and reduce intra-cluster differences to improve the performance of downstream tasks, thus solving challenge (iii).

We integrate the spectral loss in Eq. (\ref{eq13}), the node-level consistency constraint in Eq. (\ref{eq15}), with the cluster-level consistency constraint in Eq. (\ref{eq17}) to have the objective function as:
\begin{equation}
    \mathcal{J} = \mathcal{L}_{sp} + \mu \mathcal{L}_{nc} + \delta \mathcal{L}_{cc},
    \label{eq18}
\end{equation}
where $\mu$ and $\delta$ are non-negative parameters. Finally, we concatenate node representations $\mathbf{Z}$ with heterogeneous representations $\tilde{\mathbf{Z}}$ to obtain representations for downstream tasks. Actually, for the learned representations, we have the following Theorem, whose proof can be found in Appendix \ref{proof26}.

\begin{theorem}
The proposed method with dual  consistency constraints achieves a lower boundary of the model complexity $C$ and a higher generalization ability boundary $G$ than previous SHGL with the node-level consistency constraint only, i.e.,
\begin{equation}
\inf(C_{SCHOOL}) < \inf(C_{SHGL}),\quad
\sup(G_{SCHOOL}) > \sup(G_{SHGL}),
\end{equation}
    where $\inf(\cdot)$ and $\sup(\cdot)$ indicates lower bound and upper bound, respectively.
    \label{thm4}
\end{theorem}

Theorem \ref{thm4} indicates that the representations learned by the dual consistency constraints can be theoretically proved to exhibit superior generalization ability than the representations learned by previous SHGL methods with the node-level consistency constraint only, thus are expected to perform better in different downstream tasks (verified in Section \ref{resultsanalysis}).

\begin{table*}[t]
\footnotesize
\centering
\setlength\tabcolsep{3.8pt}
\caption{Classification performance (\ie Macro-F1 and Micro-F1) on heterogeneous graph datasets.}
\begin{tabular}{lccccccccccccc}
\toprule
\multirow{2}{*}{\textbf{Method}}&\multicolumn{2}{c}{\textbf{ACM}}& \multicolumn{2}{c}{\textbf{Yelp}}& \multicolumn{2}{c}{\textbf{DBLP}}& \multicolumn{2}{c}{\textbf{Aminer}} \\
\cmidrule(r){2-3} \cmidrule(r){4-5} \cmidrule(r){6-7} \cmidrule(r){8-9}
&Macro-F1&Micro-F1&Macro-F1&Micro-F1 &Macro-F1&Micro-F1 &Macro-F1&Micro-F1\\
\midrule
DeepWalk& 73.9$\pm$0.3 & 74.1$\pm$0.1 & 68.7$\pm$1.1 &73.2$\pm$0.9  &88.1$\pm$0.2 &89.5$\pm$0.3 &54.7$\pm$0.8&59.7$\pm$0.7\\
GCN & 86.9$\pm$0.2 & 87.0$\pm$0.3 & 85.0$\pm$0.6 &87.4$\pm$0.8  &90.2$\pm$0.2 &90.9$\pm$0.5 &64.5$\pm$0.7&71.5$\pm$0.9\\
GAT & 85.0$\pm$0.4 & 84.9$\pm$0.3 & 86.4$\pm$0.5 &88.2$\pm$0.7  &91.0$\pm$0.4 &92.1$\pm$0.2 &63.8$\pm$0.4&70.6$\pm$0.7\\
\midrule
Mp2vec&  87.6$\pm$0.5 &  88.1$\pm$0.3 &  78.2$\pm$0.8 & 83.6$\pm$0.9  & 85.7$\pm$0.3 & 87.6$\pm$0.6 & 58.7$\pm$0.5& 65.3$\pm$0.6\\
HAN& 89.4$\pm$0.2 & 89.2$\pm$0.2 & 90.5$\pm$1.2 &90.7$\pm$1.4  &91.2$\pm$0.4 &92.0$\pm$0.5 &65.3$\pm$0.7&72.8$\pm$0.4\\
HGT& 91.5$\pm$0.7 & 91.6$\pm$0.6 & 89.9$\pm$0.5 &90.2$\pm$0.6  &90.9$\pm$0.6 &91.7$\pm$0.8 &64.5$\pm$0.5&71.0$\pm$0.7\\
DMGI& 89.8$\pm$0.1 & 89.8$\pm$0.1 & 82.9$\pm$0.8 &85.8$\pm$0.9  &92.1$\pm$0.2 &92.9$\pm$0.3&63.8$\pm$0.4&67.6$\pm$0.5\\
DMGIattn& 88.7$\pm$0.3 & 88.7$\pm$0.5 & 82.8$\pm$0.7 &85.4$\pm$0.5  &90.9$\pm$0.2 &91.8$\pm$0.3 &62.4$\pm$0.9&66.8$\pm$0.8\\
HDMI& 90.1$\pm$0.3 & 90.1$\pm$0.3 & 80.7$\pm$0.6 &84.0$\pm$0.9  &91.3$\pm$0.2 &92.2$\pm$0.5 &65.9$\pm$0.4&71.7$\pm$0.6\\
HeCo& 88.3$\pm$0.3 & 88.2$\pm$0.2 & 85.3$\pm$0.7 &87.9$\pm$0.6  &91.0$\pm$0.3 &91.6$\pm$0.2 &71.8$\pm$0.9&78.6$\pm$0.7\\
HGCML& 90.6$\pm$0.7 & 90.7$\pm$0.5 & 90.7$\pm$0.8 &91.0$\pm$0.7  &91.9$\pm$0.8 &93.2$\pm$0.7 &70.5$\pm$0.4&76.3$\pm$0.6\\
CPIM& 91.4$\pm$0.3 & 91.3$\pm$0.2 & 90.2$\pm$0.5 &90.3$\pm$0.4  &93.2$\pm$0.6 &93.8$\pm$0.8 &70.1$\pm$0.9 &75.8$\pm$1.1\\
HGMAE& 90.5$\pm$0.5 & 90.6$\pm$0.7 & 90.5$\pm$0.7 &90.7$\pm$0.5  &92.9$\pm$0.5 &93.4$\pm$0.6 &72.3$\pm$0.9 &80.3$\pm$1.2\\
HERO& 92.2$\pm$0.5 &92.1$\pm$0.7  & 92.4$\pm$0.7&92.3$\pm$0.6 & 93.8$\pm$0.6 &94.4$\pm$0.4 & 75.1$\pm$0.7 &84.5$\pm$0.9\\
SCHOOL& \textbf{92.7$\pm$0.6} & \textbf{92.6$\pm$0.5} & \textbf{93.0$\pm$0.7} &\textbf{92.8$\pm$0.4}  &\textbf{94.0$\pm$0.3} &\textbf{94.7$\pm$0.4} &\textbf{77.5$\pm$0.9} &\textbf{86.8$\pm$0.7}\\
\bottomrule
\end{tabular}
\label{tabnode}
\end{table*}

\section{Experiments}
\label{experiments}
In this section,  we conduct experiments on both heterogeneous and homogeneous  graph datasets to evaluate the proposed method in terms of different downstream tasks (\ie node classification and node clustering), compared to  both heterogeneous and homogeneous graph methods. Detailed  settings are shown in Appendix \ref{settings}, and additional  results are shown in Appendix \ref{additional_results}.



\subsection{Experimental Setup}
\subsubsection{Datasets}
The used datasets include four heterogeneous graph datasets and two homogeneous graph datasets. Heterogeneous graph datasets include three academic datasets (\ie ACM \cite{WangJSWYCY19}, DBLP \cite{WangJSWYCY19}, and Aminer \cite{kddHuFS19}), and one business dataset (\ie Yelp \cite{lu2019relation}). Homogeneous graph datasets include two sale datasets (\ie Photo and Computers \cite{shchur2018pitfalls}).

\subsubsection{Comparison Methods}
The comparison methods include eleven heterogeneous graph methods and twelve homogeneous graph methods. The former includes two semi-supervised methods (\ie HAN \cite{WangJSWYCY19} and HGT \cite{wwwHuDWS20}), one traditional unsupervised method (\ie Mp2vec \cite{metapath2veckddDongCS17}), and eight self-supervised methods (\ie DMGI \cite{DMGIParkK0Y20}, DMGIattn \cite{DMGIParkK0Y20}, HDMI \cite{jing2021hdmi}, HeCo \cite{WangLHS21}, HGCML \cite{wang2023heterogeneous}, CPIM \cite{CPIM}, HGMAE \cite{HGMAE}, and HERO \cite{mo2024selfsupervised}). The latter includes two semi-supervised methods (\ie GCN \cite{kipf2017semisupervised} and GAT \cite{velickovic2018graph}), one traditional unsupervised method (\ie DeepWalk \cite{perozzi2014deepwalk}),  and nine self-supervised methods, (\ie DGI \cite{VelickovicFHLBH19}, GMI \cite{peng2020graph}, MVGRL \cite{hassani2020contrastive}, GRACE \cite{zhu2020deep}, GCA \cite{zhu2021graph}, G-BT \cite{GBT_KBS}, COSTA \cite{zhang2022costa}, DSSL \cite{xiaodecoupled}, and LRD \cite{yang2024self}).

For a fair comparison, we follow \cite{metapath2veckddDongCS17,WangJSWYCY19,lu2019relation,lv2021we} to select meta-paths for previous meta-path-based SHGL methods.  Moreover, we follow  \cite{DMG_ICML} to implement homogeneous graph methods on heterogeneous graph datasets by separately learning the representations of each meta-path-based graph and further concatenating them for downstream tasks. 
In addition, we replace the heterogeneous encoder $f_\theta$ with GCN to implement the proposed method on homogeneous graph datasets because there is only one node type in the homogeneous graph. Moreover, we follow previous works \cite{zhu2021graph} to generate two different views for the homogeneous graph by removing edges and masking features. The code of the proposed method is released at \url{https://github.com/YujieMo/SCHOOL}.


\subsection{Results Analysis}
\label{resultsanalysis}
\subsubsection{Effectiveness on Heterogeneous and Homogeneous Graph}
We first evaluate the effectiveness of the proposed method on the heterogeneous graph datasets and report the results of node classification and node clustering in Table \ref{tabnode} and  Appendix \ref{additional_results}, respectively.  Obviously, the proposed method obtains better performance on both node classification and node clustering tasks than comparison methods.

Specifically, first, for the node classification task, the proposed method consistently outperforms the comparison methods by large margins. For example, the proposed method on average, improves by 1.1\%, compared to the best SHGL method (\ie HERO), on four heterogeneous graph datasets. The reason can be attributed to the fact that the proposed method adaptively learns a rank-constrained affinity matrix to mitigate noisy connections among different classes, thus reducing intra-class differences. 
Second, for the node clustering task, the proposed method also obtains promising improvements. For example, the proposed method on average, improves by 3.1\%, compared to the best SHGL method (\ie HGMAE), on four heterogeneous graph datasets. This demonstrates the superiority of the proposed method, which simulates the spectral clustering with the spectral loss and conducts the cluster-level consistency constraint to further utilize the clustering information. As a result, the effectiveness of the proposed method is verified on different downstream tasks.

We further evaluate the effectiveness of the proposed method on homogeneous graph datasets and report the results of node classification in Appendix \ref{additional_results}. 
We can observe that the proposed method also achieves competitive results on the homogeneous graph datasets compared to other homogeneous graph methods.
For example, the proposed method outperforms the best self-supervised method (\ie LRD),  on almost all homogeneous graph datasets. This indicates that the proposed method is also available to learn the noise-free affinity matrix for homogeneous graphs as well as capture invariant and clustering information to benefit downstream tasks. Therefore, the effectiveness of the proposed method is verified on both heterogeneous and homogeneous graph datasets.

\begin{table*}[t]
\footnotesize
\centering
\setlength\tabcolsep{3.2pt}
\caption{Classification performance (\ie Macro-F1 and Micro-F1) of each component in the objective function $\mathcal{J}$ on all heterogeneous graph datasets.}
\begin{tabular}{cccccccccccccc}
\toprule
\multirow{2}{*}{$\mathcal{L}_{sp}$}&\multirow{2}{*}{$\mathcal{L}_{nc}$}&\multirow{2}{*}{$\mathcal{L}_{cc}$}&\multicolumn{2}{c}{\textbf{ACM}}& \multicolumn{2}{c}{\textbf{Yelp}}& \multicolumn{2}{c}{\textbf{DBLP}}& \multicolumn{2}{c}{\textbf{Aminer}} \\
\cmidrule(r){4-5} \cmidrule(r){6-7} \cmidrule(r){8-9} \cmidrule(r){10-11}
&&&Macro-F1&Micro-F1&Macro-F1&Micro-F1 &Macro-F1&Micro-F1 &Macro-F1&Micro-F1\\
\midrule
$-$&$-$&$\checkmark$&  85.9$\pm$0.3 &85.8$\pm$0.6  &91.8$\pm$0.6 &91.3$\pm$0.5 &91.0$\pm$0.2 &92.1$\pm$0.4&66.8$\pm$0.7 &75.5$\pm$0.9\\
$-$&$\checkmark$&$-$&  88.8$\pm$0.6 &88.6$\pm$0.7  &92.5$\pm$0.7 &92.1$\pm$0.4 &91.7$\pm$0.4 &92.7$\pm$0.5&72.4$\pm$0.5 &80.3$\pm$0.7 \\
$\checkmark$&$-$&$-$&  87.6$\pm$0.3 &87.5$\pm$0.5  &92.3$\pm$0.8 &92.0$\pm$0.6 &90.7$\pm$0.6 &91.7$\pm$0.6&67.3$\pm$0.6 &74.7$\pm$0.5 \\
$-$&$\checkmark$&$\checkmark$& 86.9$\pm$0.7 &86.7$\pm$0.5  &92.1$\pm$0.3 &91.5$\pm$0.4 &93.4$\pm$0.8 &94.2$\pm$0.6&75.2$\pm$0.4 &83.9$\pm$0.7\\
$\checkmark$&$-$&$\checkmark$& 89.0$\pm$0.5 &88.9$\pm$0.4  &92.4$\pm$0.5 &92.0$\pm$0.3 &93.5$\pm$0.6 &94.2$\pm$0.4&76.2$\pm$0.5 &85.2$\pm$0.8\\
$\checkmark$&$\checkmark$&$-$&  88.9$\pm$0.7 &88.8$\pm$0.6  &92.6$\pm$0.6 &92.3$\pm$0.5 &91.9$\pm$0.7 &92.8$\pm$0.8&77.1$\pm$0.8 &86.0$\pm$0.6\\
$\checkmark$&$\checkmark$&$\checkmark$&  \textbf{92.7$\pm$0.6} & \textbf{92.6$\pm$0.5} & \textbf{93.0$\pm$0.7} &\textbf{92.8$\pm$0.4}  &\textbf{94.0$\pm$0.3} &\textbf{94.7$\pm$0.4} &\textbf{77.5$\pm$0.9} &\textbf{86.8$\pm$0.7}\\
\bottomrule
\end{tabular}
\label{tabablation}
\end{table*}

\subsubsection{Ablation Study}
The proposed method investigates the objective function $\mathcal{J}$ to learn the rank-constrained affinity matrix, as well as capture invariant and clustering information. To verify the effectiveness of each component of $\mathcal{J}$ (\ie $\mathcal{L}_{sp}$, $\mathcal{L}_{nc}$, and $\mathcal{L}_{cc}$), we investigate the performance of all variants on the node classification task and report the results in Table \ref{tabablation}.

From Table \ref{tabablation}, we have the observations as follows. First, the proposed method with the complete objective function obtains the best performance. For example, the proposed method on average improves by 1.8\%,  compared to the best variant (\ie without $\mathcal{L}_{nc}$), indicating that all components in the objective function are necessary for the proposed method. This is consistent with our claims, \ie it is essential to optimize the adaptive graph structure to mitigate noisy connections as well as utilize the cluster-level information to benefit downstream tasks. Second, the variant without $\mathcal{L}_{sp}$ achieves inferior results to the other two variants (\ie without $\mathcal{L}_{nc}$ and without $\mathcal{L}_{cc}$, respectively). This can be attributed to the fact that the spectral loss $\mathcal{L}_{sp}$ enforces the cluster assignment matrix to fit the eigenvectors, which is necessary for the closed-form solution of the affinity matrix.


\subsubsection{Visualization}
\label{visualization}
To verify the effectiveness of the learned affinity matrix and the representations for downstream tasks, we visualize the affinity matrix in the heatmap and visualize the representations with t-SNE \cite{van2008visualizing} on DBLP and Aminer datasets and report the results in Figure \ref{visandcase}.

Specifically, we randomly sample 50 nodes in each class and then visualize elements of the affinity matrix $\mathbf{S}$ among sampled nodes with the heatmap, where rows and columns are reordered by node labels. In the correlation map, the darker a pixel, the larger the value of the element of $\mathbf{S}$. In Figures \ref{heatmap_1a} and \ref{heatmap_1c}, the heatmaps exhibit that there are nearly $c$  (\ie the number of classes) components in the affinity matrix, and almost all elements with large values fall in the block diagonal structure. This indicates that the affinity matrix indeed contains $c$ connected components to mitigate noisy connections among different classes. Moreover, the t-SNE visualization in  Figures \ref{heatmap_1b} and \ref{aminer} further indicate that the learned representations can be well divided into $c$ partitions. This is consistent with the observation in Theorem \ref{thm3} and verifies the effectiveness of the learned representations.

\begin{figure*}[t]
\centering
        \subfigure[$\mathbf{S}$ of DBLP]{\scalebox{0.15}{\includegraphics{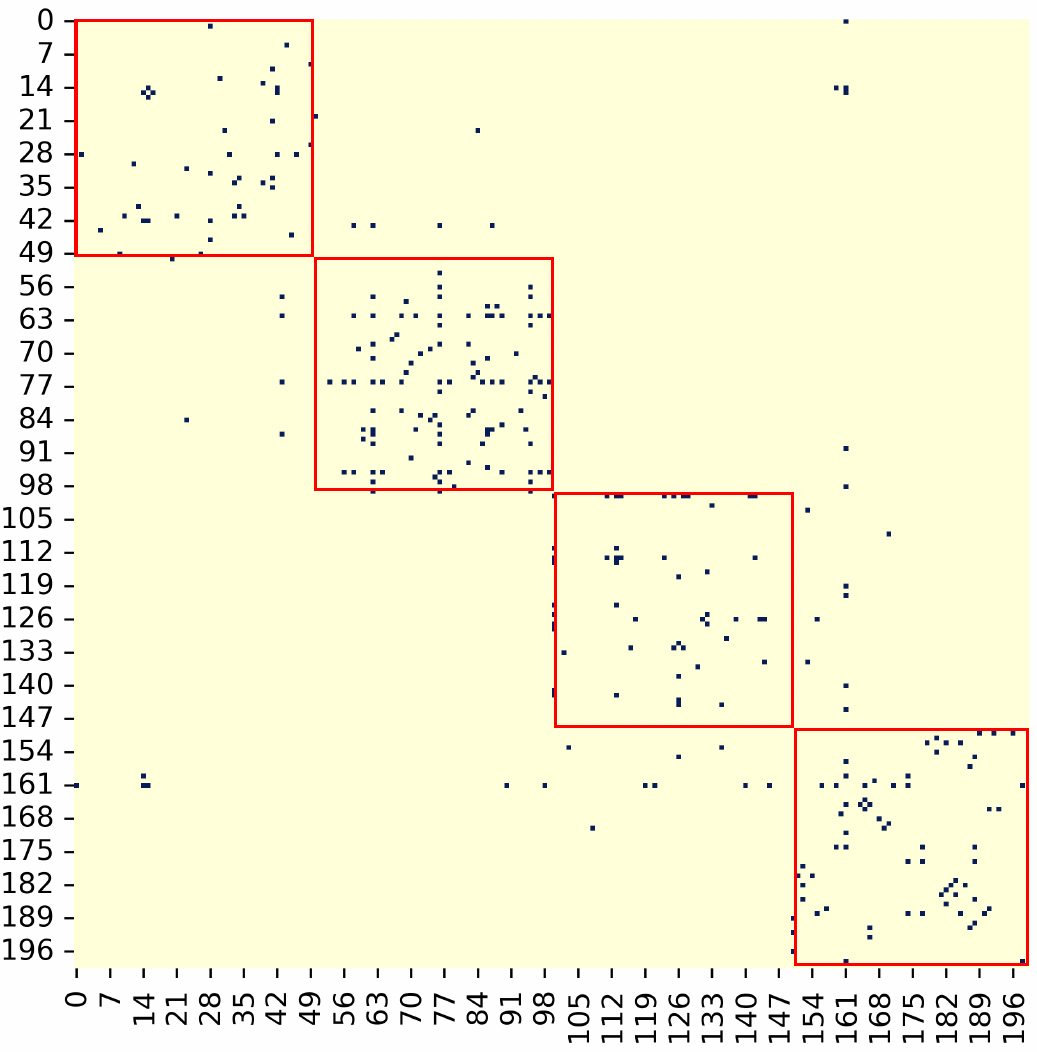}  
        }\label{heatmap_1a}}
        \subfigure[t-SNE of DBLP]{\scalebox{0.25}
{\includegraphics{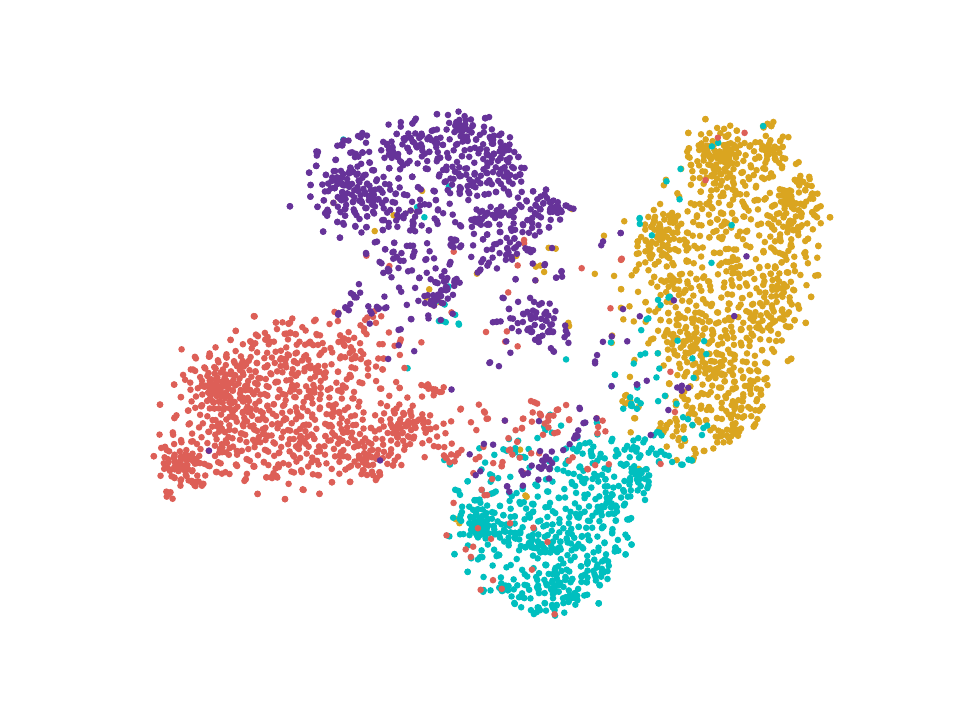}
        }\label{heatmap_1b}}
        \subfigure[$\mathbf{S}$ of Aminer]{\scalebox{0.15}
        {\includegraphics{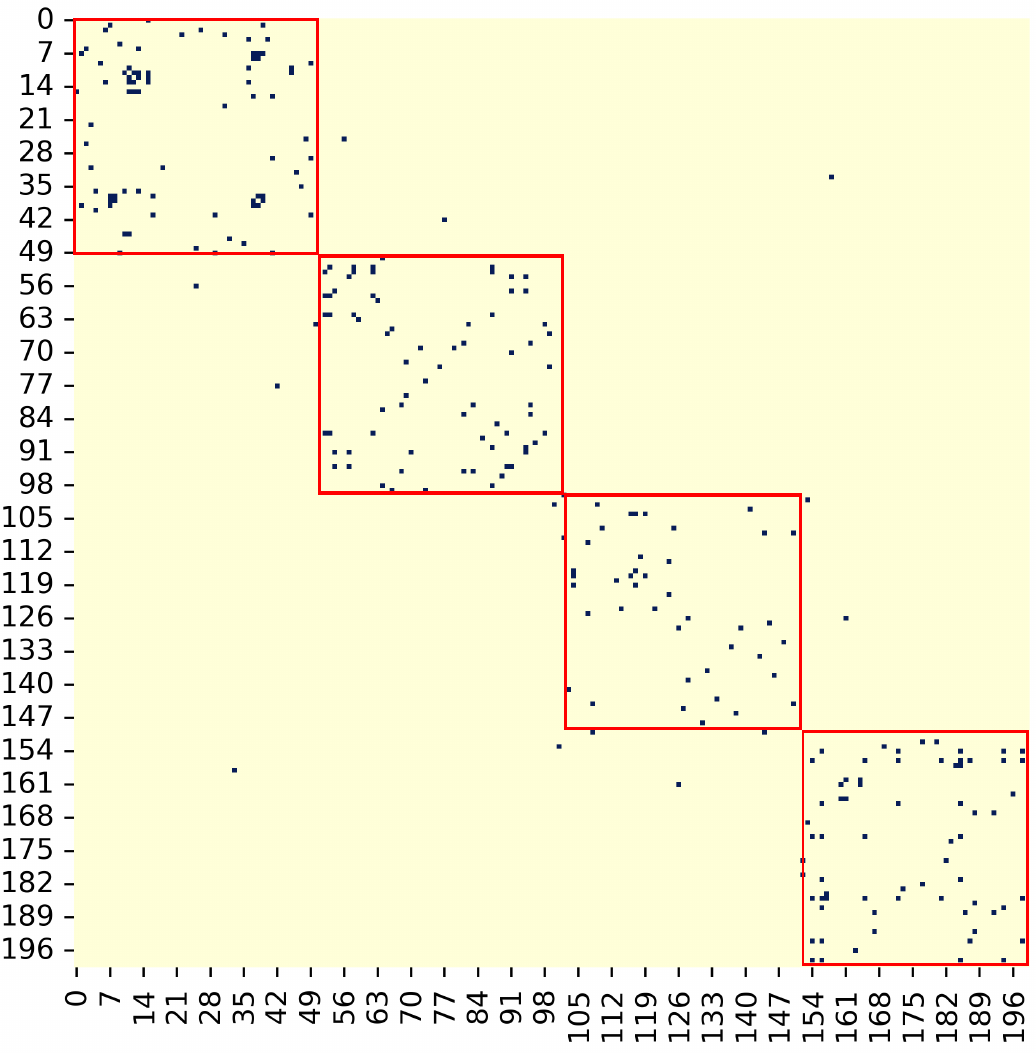}
        }\label{heatmap_1c}}
        \subfigure[t-SNE of Aminer]{\scalebox{0.25}
{\includegraphics{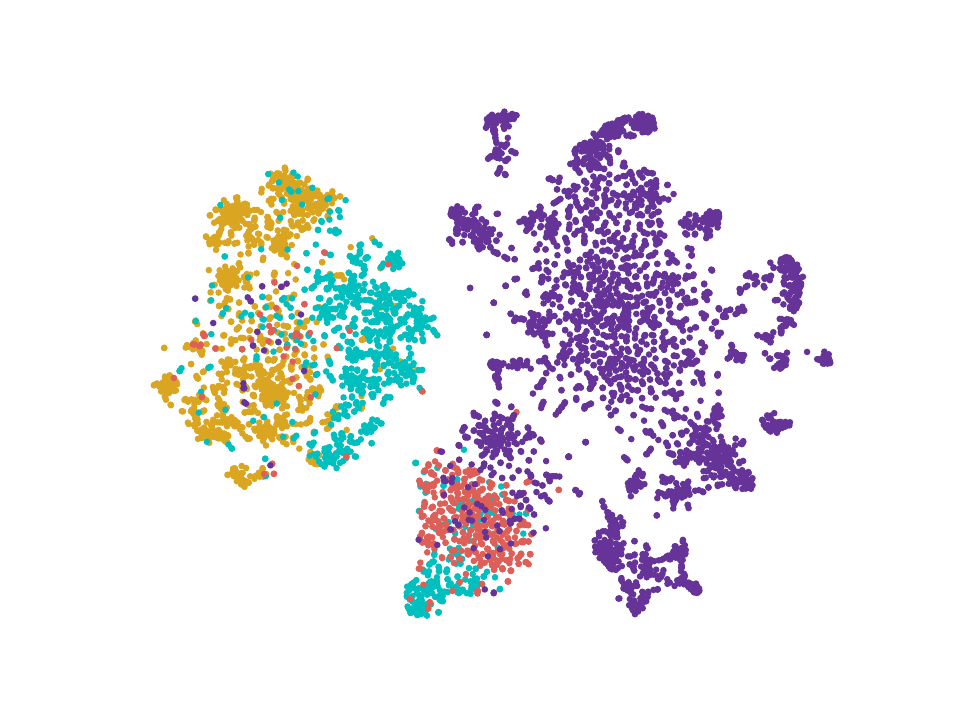}
        }\label{aminer}}
		\caption{Visualization of the affinity matrix $\mathbf{S}$ and t-SNE on DBLP and Aminer datasets. }
 \label{visandcase}
\end{figure*}

\section{Conclusion}
In this paper, we revisited previous SHGL methods from the perspective of spectral clustering and then introduced a novel framework to alleviate existing issues. Specifically, we first proved that optimizing previous SHGL methods is equivalent to performing spectral clustering with additional regularization under the orthogonalization assumption. Then we proposed an efficient spectral clustering method with the rank constraint to learn an adaptive affinity matrix and mitigate noisy connections in previous methods. Moreover, we designed node-level and cluster-level consistency constraints to capture invariant and clustering information, thus benefiting the performance of downstream tasks. Theoretical analysis indicates that the learned representations are divided into distinct partitions based on the number of classes, and are expected to achieve better generalization ability than representations of previous SHGL methods. Comprehensive experiments verify the effectiveness of the proposed method on both homogeneous and heterogeneous graph datasets on different downstream tasks. 

\textbf{Potential limitations and broader impact.}
\label{limitations}
Our potential limitation is that this work is designed based on node features. However, in heterogeneous graphs, instances arise where nodes are devoid of features. While one-hot vectors or structural embeddings can be designated as node features to tackle this problem, we recognize the necessity of devising dedicated techniques tailored for heterogeneous graphs with missing node features. In addition, the proposed method can also be used to deal with the homophily problem, which aims to explore the connections within the same class. We consider these aspects as potential directions for future research. Despite the great development of SHGL, some theoretical foundations are still lacking. Our work theoretically connects existing SHGL methods and spectral clustering and may open a new path to understanding and designing SHGL. Besides that, we do not foresee any direct negative impacts on the society.

\section*{Acknowledgments}
This project is supported by the National Key Research and Development Program of China under Grant No. 2022YFA1004100, the Natural Science Foundation of Guangdong Province of China under Grant No. 2024A1515011381, and the National Research Foundation, Singapore, under its AI Singapore Programme (AISG Award No: AISG2-RP-2021-023).


\bibliography{ref}{}

\begin{thebibliography}{10}

\bibitem{bai2020sparse}
Liang Bai and Jiye Liang.
\newblock Sparse subspace clustering with entropy-norm.
\newblock In {\em ICML}, pages 561--568, 2020.

\bibitem{GBT_KBS}
Piotr Bielak, Tomasz Kajdanowicz, and Nitesh~V. Chawla.
\newblock Graph barlow twins: {A} self-supervised representation learning framework for graphs.
\newblock {\em Knowledge-Based Systems}, 256:109631, 2022.

\bibitem{boyd2006convex}
S~Boyd, L~Vandenberghe, and L~Faybusovich.
\newblock Convex optimization.
\newblock {\em IEEE Transactions on Automatic Control}, 51(11):1859--1859, 2006.

\bibitem{metapath2veckddDongCS17}
Yuxiao Dong, Nitesh~V. Chawla, and Ananthram Swami.
\newblock metapath2vec: Scalable representation learning for heterogeneous networks.
\newblock In {\em SIGKDD}, pages 135--144, 2017.

\bibitem{fan1949theorem}
Ky~Fan.
\newblock On a theorem of weyl concerning eigenvalues of linear transformations i.
\newblock {\em Proceedings of the National Academy of Sciences}, 35(11):652--655, 1949.

\bibitem{gander1980algorithms}
Walter Gander.
\newblock Algorithms for the qr decomposition.
\newblock {\em Res. Rep}, 80(02):1251--1268, 1980.

\bibitem{haochen2021provable}
Jeff~Z HaoChen, Colin Wei, Adrien Gaidon, and Tengyu Ma.
\newblock Provable guarantees for self-supervised deep learning with spectral contrastive loss.
\newblock In {\em NeurIPS}, pages 5000--5011, 2021.

\bibitem{hassani2020contrastive}
Kaveh Hassani and Amir~Hosein Khasahmadi.
\newblock Contrastive multi-view representation learning on graphs.
\newblock In {\em ICML}, pages 4116--4126, 2020.

\bibitem{he2022analyzing}
Dongxiao He, Chundong Liang, Cuiying Huo, Zhiyong Feng, Di~Jin, Liang Yang, and Weixiong Zhang.
\newblock Analyzing heterogeneous networks with missing attributes by unsupervised contrastive learning.
\newblock {\em IEEE Transactions on Neural Networks and Learning Systems}, 2022.

\bibitem{he2023contrastive}
Dongxiao He, Jitao Zhao, Rui Guo, Zhiyong Feng, Di~Jin, Yuxiao Huang, Zhen Wang, and Weixiong Zhang.
\newblock Contrastive learning meets homophily: two birds with one stone.
\newblock In {\em ICML}, pages 12775--12789, 2023.

\bibitem{kddHuFS19}
Binbin Hu, Yuan Fang, and Chuan Shi.
\newblock Adversarial learning on heterogeneous information networks.
\newblock In {\em SIGKDD}, pages 120--129, 2019.

\bibitem{hu2018leveraging}
Binbin Hu, Chuan Shi, Wayne~Xin Zhao, and Philip~S Yu.
\newblock Leveraging meta-path based context for top-n recommendation with a neural co-attention model.
\newblock In {\em SIGKDD}, pages 1531--1540, 2018.

\bibitem{wwwHuDWS20}
Ziniu Hu, Yuxiao Dong, Kuansan Wang, and Yizhou Sun.
\newblock Heterogeneous graph transformer.
\newblock In {\em WWW}, pages 2704--2710, 2020.

\bibitem{huang2020deep}
Jiabo Huang, Shaogang Gong, and Xiatian Zhu.
\newblock Deep semantic clustering by partition confidence maximisation.
\newblock In {\em CVPR}, pages 8849--8858, 2020.

\bibitem{huangmultiplex}
Yudi Huang, Yujie Mo, Yujing Liu, Ci~Nie, Guoqiu Wen, and Xiaofeng Zhu.
\newblock Multiplex graph representation learning via bi-level optimization.
\newblock In {\em IJCAI}, 2024.

\bibitem{jiang2021methods}
Yiding Jiang, Parth Natekar, Manik Sharma, Sumukh~K Aithal, Dhruva Kashyap, Natarajan Subramanyam, Carlos Lassance, Daniel~M Roy, Gintare~Karolina Dziugaite, Suriya Gunasekar, et~al.
\newblock Methods and analysis of the first competition in predicting generalization of deep learning.
\newblock In {\em NeurIPS}, pages 170--190, 2021.

\bibitem{jin2023dual}
Di~Jin, Luzhi Wang, Yizhen Zheng, Guojie Song, Fei Jiang, Xiang Li, Wei Lin, and Shirui Pan.
\newblock Dual intent enhanced graph neural network for session-based new item recommendation.
\newblock In {\em WWW}, pages 684--693, 2023.

\bibitem{jing2021hdmi}
Baoyu Jing, Chanyoung Park, and Hanghang Tong.
\newblock Hdmi: High-order deep multiplex infomax.
\newblock In {\em WWW}, pages 2414--2424, 2021.

\bibitem{kingma2015adam}
P.~Diederik Kingma and Lei~Jimmy Ba.
\newblock Adam: A method for stochastic optimization.
\newblock In {\em ICLR}, 2015.

\bibitem{kipf2017semisupervised}
N.~Thomas Kipf and Max Welling.
\newblock Semi-supervised classification with graph convolutional networks.
\newblock In {\em ICLR}, pages 1--14, 2017.

\bibitem{law2017deep}
Marc~T Law, Raquel Urtasun, and Richard~S Zemel.
\newblock Deep spectral clustering learning.
\newblock In {\em ICML}, pages 1985--1994, 2017.

\bibitem{Liangke_Survey}
Ke~Liang, Lingyuan Meng, Meng Liu, Yue Liu, Wenxuan Tu, Siwei Wang, Sihang Zhou, Xinwang Liu, and Fuchun Sun.
\newblock A survey of knowledge graph reasoning on graph types: Static, dynamic, and multimodal.
\newblock {\em arXiv preprint arXiv:2212.05767}, 2022.

\bibitem{liu2024self}
Meng Liu, Ke~Liang, Yawei Zhao, Wenxuan Tu, Sihang Zhou, Xinbiao Gan, Xinwang Liu, and Kunlun He.
\newblock Self-supervised temporal graph learning with temporal and structural intensity alignment.
\newblock {\em IEEE Transactions on Neural Networks and Learning Systems}, 2024.

\bibitem{liu2023deep}
Meng Liu, Yue Liu, Ke~Liang, Wenxuan Tu, Siwei Wang, Sihang Zhou, and Xinwang Liu.
\newblock Deep temporal graph clustering.
\newblock In {\em ICLR}, 2024.

\bibitem{liu2022revisiting}
Nian Liu, Xiao Wang, Deyu Bo, Chuan Shi, and Jian Pei.
\newblock Revisiting graph contrastive learning from the perspective of graph spectrum.
\newblock In {\em NeurIPS}, volume~35, pages 2972--2983, 2022.

\bibitem{liu2023dink}
Yue Liu, Ke~Liang, Jun Xia, Sihang Zhou, Xihong Yang, Xinwang Liu, and Stan~Z Li.
\newblock Dink-net: Neural clustering on large graphs.
\newblock In {\em ICML}, 2023.

\bibitem{lu2019relation}
Yuanfu Lu, Chuan Shi, Linmei Hu, and Zhiyuan Liu.
\newblock Relation structure-aware heterogeneous information network embedding.
\newblock In {\em AAAI}, pages 4456--4463, 2019.

\bibitem{lv2021we}
Qingsong Lv, Ming Ding, Qiang Liu, Yuxiang Chen, Wenzheng Feng, Siming He, Chang Zhou, Jianguo Jiang, Yuxiao Dong, and Jie Tang.
\newblock Are we really making much progress? revisiting, benchmarking and refining heterogeneous graph neural networks.
\newblock In {\em SIGKDD}, pages 1150--1160, 2021.

\bibitem{DMG_ICML}
Yujie Mo, Yajie Lei, Jialie Shen, Xiaoshuang Shi, Heng~Tao Shen, and Xiaofeng Zhu.
\newblock Disentangled multiplex graph representation learning.
\newblock In {\em ICML}, pages 24983--25005, 2023.

\bibitem{mo2024selfsupervised}
Yujie Mo, Feiping Nie, Zheng Zhang, Ping Hu, Heng~Tao Shen, Xinchao Wang, and Xiaofeng Zhu.
\newblock Self-supervised heterogeneous graph learning: a homophily and heterogeneity view.
\newblock In {\em ICLR}, 2024.

\bibitem{CPIM}
Yujie Mo, Zongqian Wu, Yuhuan Chen, Xiaoshuang Shi, Heng~Tao Shen, and Xiaofeng Zhu.
\newblock Multiplex graph representation learning via common and private information mining.
\newblock In {\em AAAI}, pages 9217--9225, 2023.

\bibitem{mohar1991laplacian}
Bojan Mohar, Y~Alavi, G~Chartrand, and OR~Oellermann.
\newblock The laplacian spectrum of graphs.
\newblock {\em Graph theory, combinatorics, and applications}, 2(871-898):12, 1991.

\bibitem{natekar2020representation}
Parth Natekar and Manik Sharma.
\newblock Representation based complexity measures for predicting generalization in deep learning.
\newblock {\em arXiv preprint arXiv:2012.02775}, 2020.

\bibitem{nie2014clustering}
Feiping Nie, Xiaoqian Wang, and Heng Huang.
\newblock Clustering and projected clustering with adaptive neighbors.
\newblock In {\em SIGKDD}, pages 977--986, 2014.

\bibitem{nie2020self}
Feiping Nie, Danyang Wu, Rong Wang, and Xuelong Li.
\newblock Self-weighted clustering with adaptive neighbors.
\newblock {\em IEEE Transactions on Neural Networks and Learning Systems}, pages 3428--3441, 2020.

\bibitem{oord2018representation}
Aaron van~den Oord, Yazhe Li, and Oriol Vinyals.
\newblock Representation learning with contrastive predictive coding.
\newblock {\em arXiv preprint arXiv:1807.03748}, 2018.

\bibitem{pan2021multi}
Erlin Pan and Zhao Kang.
\newblock Multi-view contrastive graph clustering.
\newblock In {\em NeurIPS}, volume~34, pages 2148--2159, 2021.

\bibitem{DMGIParkK0Y20}
Chanyoung Park, Donghyun Kim, Jiawei Han, and Hwanjo Yu.
\newblock Unsupervised attributed multiplex network embedding.
\newblock In {\em AAAI}, pages 5371--5378, 2020.

\bibitem{peng2023grlc}
Liang Peng, Yujie Mo, Jie Xu, Jialie Shen, Xiaoshuang Shi, Xiaoxiao Li, Heng~Tao Shen, and Xiaofeng Zhu.
\newblock Grlc: Graph representation learning with constraints.
\newblock {\em IEEE Transactions on Neural Networks and Learning Systems}, 2023.

\bibitem{peng2020graph}
Zhen Peng, Wenbing Huang, Minnan Luo, Qinghua Zheng, Yu~Rong, Tingyang Xu, and Junzhou Huang.
\newblock Graph representation learning via graphical mutual information maximization.
\newblock In {\em WWW}, pages 259--270, 2020.

\bibitem{perozzi2014deepwalk}
Bryan Perozzi, Rami Al-Rfou, and Steven Skiena.
\newblock Deepwalk: Online learning of social representations.
\newblock In {\em SIGKDD}, pages 701--710, 2014.

\bibitem{shaham2018spectralnet}
Uri Shaham, Kelly Stanton, Henry Li, Ronen Basri, Boaz Nadler, and Yuval Kluger.
\newblock Spectralnet: Spectral clustering using deep neural networks.
\newblock In {\em ICLR}, 2018.

\bibitem{shchur2018pitfalls}
Oleksandr Shchur, Maximilian Mumme, Aleksandar Bojchevski, and Stephan G{\"u}nnemann.
\newblock Pitfalls of graph neural network evaluation.
\newblock {\em arXiv preprint arXiv:1811.05868}, 2018.

\bibitem{shi2018heterogeneous}
Chuan Shi, Binbin Hu, Wayne~Xin Zhao, and S~Yu Philip.
\newblock Heterogeneous information network embedding for recommendation.
\newblock {\em IEEE Transactions on Knowledge and Data Engineering}, 31(2):357--370, 2018.

\bibitem{shi2016survey}
Chuan Shi, Yitong Li, Jiawei Zhang, Yizhou Sun, and S~Yu Philip.
\newblock A survey of heterogeneous information network analysis.
\newblock {\em IEEE Transactions on Knowledge and Data Engineering}, pages 17--37, 2016.

\bibitem{tan2023contrastive}
Zhiquan Tan, Yifan Zhang, Jingqin Yang, and Yang Yuan.
\newblock Contrastive learning is spectral clustering on similarity graph.
\newblock In {\em ICLR}, 2024.

\bibitem{tang2022unified}
Chang Tang, Zhenglai Li, Jun Wang, Xinwang Liu, Wei Zhang, and En~Zhu.
\newblock Unified one-step multi-view spectral clustering.
\newblock {\em IEEE Transactions on Knowledge and Data Engineering}, 35(6):6449--6460, 2022.

\bibitem{HGMAE}
Yijun Tian, Kaiwen Dong, Chunhui Zhang, Chuxu Zhang, and Nitesh~V. Chawla.
\newblock Heterogeneous graph masked autoencoders.
\newblock In {\em AAAI}, pages 9997--10005, 2023.

\bibitem{van2008visualizing}
Laurens Van~der Maaten and Geoffrey Hinton.
\newblock Visualizing data using t-sne.
\newblock {\em Journal of Machine Learning Research}, 9(11), 2008.

\bibitem{velickovic2018graph}
Petar Velickovic, Guillem Cucurull, Arantxa Casanova, Adriana Romero, Pietro Liò, and Yoshua Bengio.
\newblock Graph attention networks.
\newblock In {\em ICLR}, pages 1--12, 2018.

\bibitem{VelickovicFHLBH19}
Petar Velickovic, William Fedus, William~L. Hamilton, Pietro Li{\`{o}}, Yoshua Bengio, and R.~Devon Hjelm.
\newblock Deep graph infomax.
\newblock In {\em ICLR}, pages 1--17, 2019.

\bibitem{von2007tutorial}
Ulrike Von~Luxburg.
\newblock A tutorial on spectral clustering.
\newblock {\em Statistics and computing}, 17:395--416, 2007.

\bibitem{wang2018spectral}
Qi~Wang, Zequn Qin, Feiping Nie, and Xuelong Li.
\newblock Spectral embedded adaptive neighbors clustering.
\newblock {\em IEEE Transactions on Neural Networks and Learning Systems}, pages 1265--1271, 2018.

\bibitem{wang2003image}
Song Wang and Jeffrey~Mark Siskind.
\newblock Image segmentation with ratio cut.
\newblock {\em IEEE Transactions on Pattern Analysis and Machine Intelligence}, 25(6):675--690, 2003.

\bibitem{wang2022survey}
Xiao Wang, Deyu Bo, Chuan Shi, Shaohua Fan, Yanfang Ye, and S~Yu Philip.
\newblock A survey on heterogeneous graph embedding: methods, techniques, applications and sources.
\newblock {\em IEEE Transactions on Big Data}, 9(2):415--436, 2022.

\bibitem{WangJSWYCY19}
Xiao Wang, Houye Ji, Chuan Shi, Bai Wang, Yanfang Ye, Peng Cui, and Philip~S. Yu.
\newblock Heterogeneous graph attention network.
\newblock In {\em WWW}, pages 2022--2032, 2019.

\bibitem{WangLHS21}
Xiao Wang, Nian Liu, Hui Han, and Chuan Shi.
\newblock Self-supervised heterogeneous graph neural network with co-contrastive learning.
\newblock In {\em SIGKDD}, pages 1726--1736, 2021.

\bibitem{wang2023heterogeneous}
Zehong Wang, Qi~Li, Donghua Yu, Xiaolong Han, Xiao-Zhi Gao, and Shigen Shen.
\newblock Heterogeneous graph contrastive multi-view learning.
\newblock In {\em SDM}, pages 136--144, 2023.

\bibitem{wu2023molformer}
Fang Wu, Dragomir Radev, and Stan~Z Li.
\newblock Molformer: Motif-based transformer on 3d heterogeneous molecular graphs.
\newblock In {\em AAAI}, volume~37, pages 5312--5320, 2023.

\bibitem{wu2021self}
Lirong Wu, Haitao Lin, Cheng Tan, Zhangyang Gao, and Stan~Z Li.
\newblock Self-supervised learning on graphs: Contrastive, generative, or predictive.
\newblock {\em IEEE Transactions on Knowledge and Data Engineering}, pages 1--20, 2021.

\bibitem{xiaodecoupled}
Teng Xiao, Zhengyu Chen, Zhimeng Guo, Zeyang Zhuang, and Suhang Wang.
\newblock Decoupled self-supervised learning for graphs.
\newblock In {\em NeurIPS}, 2022.

\bibitem{yang2022self}
Liang Yang, Cheng Chen, Weixun Li, Bingxin Niu, Junhua Gu, Chuan Wang, Dongxiao He, Yuanfang Guo, and Xiaochun Cao.
\newblock Self-supervised graph neural networks via diverse and interactive message passing.
\newblock In {\em AAAI}, pages 4327--4336, 2022.

\bibitem{yang2024self}
Liang Yang, Runjie Shi, Qiuliang Zhang, Zhen Wang, Xiaochun Cao, Chuan Wang, et~al.
\newblock Self-supervised graph neural networks via low-rank decomposition.
\newblock In {\em NeurIPS}, 2024.

\bibitem{DealMVC}
Xihong Yang, Jin Jiaqi, Siwei Wang, Ke~Liang, Yue Liu, Yi~Wen, Suyuan Liu, Sihang Zhou, Xinwang Liu, and En~Zhu.
\newblock Dealmvc: Dual contrastive calibration for multi-view clustering.
\newblock In {\em ACM MM}, pages 337--346, 2023.

\bibitem{CCGC}
Xihong Yang, Yue Liu, Sihang Zhou, Siwei Wang, Wenxuan Tu, Qun Zheng, Xinwang Liu, Liming Fang, and En~Zhu.
\newblock Cluster-guided contrastive graph clustering network.
\newblock In {\em AAAI}, pages 10834--10842, 2023.

\bibitem{yang2019deep}
Xu~Yang, Cheng Deng, Feng Zheng, Junchi Yan, and Wei Liu.
\newblock Deep spectral clustering using dual autoencoder network.
\newblock In {\em CVPR}, pages 4066--4075, 2019.

\bibitem{yu2024hgprompt}
Xingtong Yu, Yuan Fang, Zemin Liu, and Xinming Zhang.
\newblock Hgprompt: Bridging homogeneous and heterogeneous graphs for few-shot prompt learning.
\newblock In {\em AAAI}, pages 16578--16586, 2024.

\bibitem{yu2022molecular}
Zhaoning Yu and Hongyang Gao.
\newblock Molecular representation learning via heterogeneous motif graph neural networks.
\newblock In {\em ICML}, pages 25581--25594, 2022.

\bibitem{zhang2022simple}
Rui Zhang, Arthur Zimek, and Peter Schneider-Kamp.
\newblock A simple meta-path-free framework for heterogeneous network embedding.
\newblock In {\em CIKM}, pages 2600--2609, 2022.

\bibitem{zhang2022costa}
Yifei Zhang, Hao Zhu, Zixing Song, Piotr Koniusz, and Irwin King.
\newblock Costa: Covariance-preserving feature augmentation for graph contrastive learning.
\newblock In {\em SIGKDD}, pages 2524--2534, 2022.

\bibitem{ZhaoWSHSY21}
Jianan Zhao, Xiao Wang, Chuan Shi, Binbin Hu, Guojie Song, and Yanfang Ye.
\newblock Heterogeneous graph structure learning for graph neural networks.
\newblock In {\em AAAI}, pages 4697--4705, 2021.

\bibitem{CKD00010YCLF022}
Sheng Zhou, Kang Yu, Defang Chen, Bolang Li, Yan Feng, and Chun Chen.
\newblock Collaborative knowledge distillation for heterogeneous information network embedding.
\newblock In {\em WWW}, pages 1631--1639, 2022.

\bibitem{zhou2020multi}
Sihang Zhou, Xinwang Liu, Jiyuan Liu, Xifeng Guo, Yawei Zhao, En~Zhu, Yongping Zhai, Jianping Yin, and Wen Gao.
\newblock Multi-view spectral clustering with optimal neighborhood laplacian matrix.
\newblock In {\em AAAI}, pages 6965--6972, 2020.

\bibitem{zhu2022structure}
Yanqiao Zhu, Yichen Xu, Hejie Cui, Carl Yang, Qiang Liu, and Shu Wu.
\newblock Structure-enhanced heterogeneous graph contrastive learning.
\newblock In {\em SDM}, pages 82--90, 2022.

\bibitem{zhu2020deep}
Yanqiao Zhu, Yichen Xu, Feng Yu, Qiang Liu, Shu Wu, and Liang Wang.
\newblock Deep graph contrastive representation learning.
\newblock {\em arXiv preprint arXiv:2006.04131}, 2020.

\bibitem{zhu2021graph}
Yanqiao Zhu, Yichen Xu, Feng Yu, Qiang Liu, Shu Wu, and Liang Wang.
\newblock Graph contrastive learning with adaptive augmentation.
\newblock In {\em WWW}, pages 2069--2080, 2021.

\end{thebibliography}
\bibliographystyle{plain}

\clearpage
\appendix

\section{Related Work}
\label{relatedwork}
This section briefly reviews topics related to this work, including self-supervised heterogeneous graph learning in Section \ref{related1}, and spectral clustering in Section \ref{related2}.

\subsection{Self-Supervised Heterogeneous Graph Learning}
\label{related1}

In recent years, self-supervised heterogeneous graph learning (SHGL) has emerged as a helpful technique to deal with the heterogeneous graph that consists of different types of entities without needing labeled data \cite{WangJSWYCY19,wu2021self,WangLHS21,wang2022survey,peng2023grlc, liu2024self, yu2024hgprompt}. As a result, SHGL captures meaningful representations of nodes and edges, enabling better performance in downstream tasks like node classification and node clustering. Due to its powerful capability, SHGL has been applied to various real applications, such as social network analysis \cite{Liangke_Survey, yang2022self, he2023contrastive}, and recommendation systems \cite{shi2018heterogeneous, jin2023dual, huangmultiplex}.

Existing SHGL methods can be broadly classified into two groups, \ie meta-path-based methods and adaptive-graph-based methods. In meta-path-based methods, several graphs are usually constructed based on different pre-defined meta-paths to examine diverse relationships among nodes that share similar labels \cite{jing2021hdmi, zhu2022structure}.
For example, STENCIL \cite{zhu2022structure} and HDMI \cite{jing2021hdmi} construct meth-path-based graphs and then conduct node-level consistency constraints (\eg contrastive loss) between node representations in different graphs. In addition, HGCML \cite{wang2023heterogeneous} and CPIM \cite{CPIM} propose to maximize the mutual information between node representations from different meta-path-based graphs. However, pre-defined meta-paths in these methods generally require expert knowledge and prohibitive computation costs \cite{zhang2022simple}. Therefore, adaptive-graph-based methods are proposed to learn the adaptive graph structures to capture the relationships among nodes that possess the same label, instead of using meta-paths.
For example, recently, HERO \cite{mo2024selfsupervised} made the first attempt to learn an adaptive self-expressive matrix to capture the homophily in the heterogeneous graph, thus avoiding meta-paths.

Although existing SHGL methods (especially the adaptive-graph-based methods) have achieved impressive performance in several tasks, the learned graph structure cannot be guaranteed optimal. As a result, the learned graph structure may contain noisy connections from different classes to affect the message-passing process and weaken the discriminative information in node representations.

\subsection{Spectral Clustering}
\label{related2}

Spectral clustering partitions data points into clusters based on a similarity matrix derived from the data \cite{wang2018spectral,tang2022unified, CCGC}. Owing to its proficiency in identifying clusters with complex shapes and handling non-linearly separable data, spectral clustering is widely used in many scenarios  \cite{zhou2020multi,liu2023deep, liu2023dink, DealMVC}. 

The spectral clustering methods can be broadly classified into two groups, \ie traditional spectral clustering and deep spectral clustering. Traditional spectral clustering methods aim to group data points that are similar to each other while being dissimilar to points in other clusters by eigendecomposition \cite{nie2014clustering, nie2020self}. For example, CAN \cite{nie2014clustering} proposes to learn the data similarity matrix and clustering structure simultaneously with the eigendecomposition.
SWCAN \cite{nie2020self} further assigns weights for different features to learn the similarity graph and partition samples into clusters simultaneously. Despite its effectiveness, traditional spectral clustering generally requires expensive computation costs, especially for large datasets. To alleviate this issue, deep spectral clustering methods have been proposed in recent years. For example, DSC \cite{yang2019deep} employs an encoder and two decoders to train the network, thus
obtaining discriminative representations for clustering and implementing the cluster assignment via the neural network. DSCL \cite{law2017deep} introduces a novel metric learning framework that leverages spectral clustering principles, thus reducing complexity to linear levels. SpectralNet \cite{shaham2018spectralnet} proposes to learn a mapping function via the orthogonalization network to address the out-of-sample-extension and scalability problems.

The above methods conduct spectral clustering explicitly. Surprisingly, recent research shows that some popular self-supervised methods also implicitly conduct spectral clustering \cite{ haochen2021provable, tan2023contrastive}. For example, \cite{haochen2021provable} demonstrates contrastive learning performs spectral clustering on the population augmentation graph by replacing the standard InfoNCE \cite{oord2018representation} with its proposed spectral contrastive loss. \cite{tan2023contrastive} demonstrates that contrastive learning with the standard InfoNCE loss is equivalent to spectral clustering on the similarity graph. Although these methods make efforts to connect previous self-supervised methods with spectral clustering, they cannot be easily transferred to SHGL. First, these methods are almost based on the augmentation graph, which assumes that different augmentations of the same sample connect each other and thus form a graph. In contrast, in SHGL, there is no augmentation, and the graph is constructed by connecting different samples. Second, compared to the above methods, SHGL incorporates the message-passing process, which makes it more complex. Therefore, connecting SHGL methods with the spectral clustering remains challenging.

\section{Algorithm and Complexity Analysis}
This section provides the pseudo-code of the proposed method in Section \ref{algo11}, and the complexity analysis of our method in Section \ref{complexity}.
\subsection{Algorithm}
\label{algo11}
\begin{algorithm}[H]
\caption{The pseudo-code of the proposed method.}
\label{algo1}
\begin{algorithmic}[1]
\REQUIRE Heterogeneous graph $\mathbf{G}=(\mathcal{V}, \mathcal{E}, \mathbf{X}, \mathcal{T}, \mathcal{R})$, non-negative parameters $\beta$, $\gamma$, $\eta$, $\mu$ and $\delta$;\\
\ENSURE Encoders $g_{\phi}$, $f_{\theta}$;

\STATE Initialize  parameters;
\WHILE{not converge}
\STATE Obtain semantic representations $\mathbf{H}$ with encoder $g_{\phi}$;\\
\STATE Obtain the closed-form solution of the affinity matrix $\mathbf{S}$  by Eq. (\ref{eq10});\\
\STATE Obtain the orthogonal cluster assignment matrix $\mathbf{Y}$  by Eq. (\ref{eq11}) and Eq. (\ref{eq12});\\
 \STATE Conduct the spectral loss based on $\mathbf{Y}$ and $\mathbf{S}$  by Eq. (\ref{eq13});\\
\STATE Obtain node representations $\mathbf{Z}$ by $\mathbf{Z} = \mathbf{S}\mathbf{H}$;\\
\STATE Obtain heterogeneous representations $\Tilde{\mathbf{Z}}$ with encoder $f_{\theta}$;\\
\STATE Project node and heterogeneous representations into a latent space to obtain $\mathbf{Q}$ and $\tilde{\mathbf{Q}}$;\\
 \STATE Conduct the node-level consistency constraint between $\mathbf{Q}$ and $\tilde{\mathbf{Q}}$  by Eq. (\ref{eq15});\\
 \STATE Obtain cluster representations $\hat{\mathbf{Q}}$ by Eq. (\ref{eq16});\\
\STATE Conduct the cluster-level consistency constraint between $\tilde{\mathbf{Q}}$ and $\hat{\mathbf{Q}}$  by Eq. (\ref{eq17});\\
\STATE Compute the objective function $\mathcal{J}$ by Eq. (\ref{eq18});\\
\STATE Back-propagate $\mathcal{J}$ to  update  model  weights;\\
\ENDWHILE\\
\end{algorithmic}
\end{algorithm}

\subsection{Complexity Analysis}
\label{complexity}
Based on the Algorithm \ref{algo1} above, we then analyze the time complexity of the proposed method. Recalling Eq. (\ref{eq10}) in the main text:
\begin{equation}
    \mathbf{s}_{ij} = (-\frac{1}{2 \bm{\alpha}} \mathbf{d}_{i} + \lambda )_+,
    \label{eqap10}
\end{equation}
where $\mathbf{d}_{ij} = \|\mathbf{h}_{i}-\mathbf{h}_{j}\|_{2}^{2} + \beta\left\|\mathbf{f}_{i}-\mathbf{f}_{j}\right\|_{2}^{2}$, where $\mathbf{H} \in \mathbb{R}^{n \times d}$ and $\mathbf{F} \in \mathbb{R}^{n \times c}$ are semantic representations and eigenvector matrix, $d$ and $c$ indicate number of dimensions and classes, and $n$ indicates the number of nodes. To reduce the computation costs, the proposed method proposes to only calculate $\mathbf{s}_{ij}$ between node $v_i$ and its $k$ nearest neighbors. Therefore, the time complexity of Eq. (\ref{eqap10}) is $\mathcal{O}(nk)$. Moreover, the proposed method proposes to replace the eigendecomposition with a projection head and orthogonalization layer  to further reduce the time complexity. Specifically, the time complexity of the orthogonal process for $\mathbf{H}$ and $\mathbf{Y}$ with the QR decomposition is $\mathcal{O}(nd^2)$ and $\mathcal{O}(nc^2)$, respectively. The time complexity of the inversion process in Eq. (\ref{eq12}) for $\mathbf{H}$ and $\mathbf{F}$ is $\mathcal{O}(d^3)$ and $\mathcal{O}(c^3)$. Moreover,  the time complexity of the spectral loss is $\mathcal{O}(nkc)$ and the time complexity of $\mathbf{Z} = \mathbf{S}\mathbf{H}$ is $\mathcal{O}(nkd)$. In addition, the time complexity of node-level and cluster-level consistency constraints are $\mathcal{O}(nd^2)$ and $\mathcal{O}(n)$, respectively. Therefore, the overall complexity of the proposed  method is $\mathcal{O}(nd^2 + nc^2 + nkd + nkc +d^3 + c^3)$ in each epoch, where $d^2, c^2 <n$, thus is scaled linearly with the sample size.

\section{Proofs of Theorems}
\label{proofs}
This section provides definition, detailed proofs of Theorems, and derivation process in Section \ref{method}, including the proofs of Theorem \ref{thm1} in Section \ref{proof22}, the proofs of Theorem \ref{thm2} in Section \ref{proof23}, the proofs of Theorem \ref{thm3} in Section \ref{proof25}, the proofs of Theorem \ref{thm4} in Section \ref{proof26}, the derivation  of Eq. (\ref{eq7}) in Section \ref{derivationeq7}, the derivation of the closed-form solution and parameters in Section \ref{closedform}, and the derivation of the orthogonalization in Section \ref{orthgo}.

\subsection{Proof of Theorem \ref{thm1}}
\label{proof22}

\begin{theorem}
{\rm(Restating Theorem \ref{thm1} in the main
text).}
Assume the learned representations $\mathbf{H}$ are orthogonal, optimizing previous meta-path-based  and adaptive-graph-based SHGL methods is equivalent to performing spectral clustering with additional regularization, i.e.,
\begin{equation}
    \min _{\mathbf{H}}\mathcal{L}_{SHGL}\cong \min _{\mathbf{H}} \operatorname{Tr}(\mathbf{H}^{T} \hat{\mathbf{L}} \mathbf{H}) + R(\mathbf{H}) \text { s.t., } \mathbf{H}^{T} \mathbf{H}=\mathbf{I} \text {, }
\end{equation}
where $R(\cdot)$ indicates the regularization term, $\hat{\mathbf{L}}$ indicates the Laplacian matrix of the meta-path-based graph or the adaptive graph structure.
\label{apthm1}
\end{theorem}

\begin{proof}
First, we prove the connection between previous meta-path-based SHGL methods and spectral clustering. To do this, take a heterogeneous graph with two meta-paths as an example, we let $\mathcal{G}= \{\mathcal{G}^{(1)}\cup  \mathcal{G}^{(2)}\}$ indicates the union of all meta-path-based graph views. Moreover, we denote the representations of previous methods before the message-passing as $\mathbf{H}$ (generally obtained by linear mapping from original node features). In addition, we denote the node representations of different graph views after the message-passing as $\mathbf{Z}^{(r)}$, respectively, where $r=1, 2$, \ie
    \begin{equation}
        \mathbf{z}_i^{(r)}=\mathbf{h}_i + \{\mathbf{h}_j, v_j\in \mathcal{N}(v_i)^{(r)}\},
        \label{apeq1}
    \end{equation}
where $\mathcal{N}(v_i)^{(r)}$ indicates the one-hop neighbors of node $v_i$ in the $r$-th meta-path-based graph.
    
Based on the node representations $\mathbf{Z}^{(r)}$ of each graph, previous meta-path-based SHGL methods generally propose to extract the invariant information among node representations from different meta-path-based graphs. Here, we take the Mean Squared Error (MSE) loss as a simple example to extract the invariance and then conduct an analysis of previous meta-path-based SHGL methods. Therefore, the objective function of previous meta-path-based SHGL methods can be formulated as:
    \begin{equation}
        \min_{\theta } \sum_{i}^{n}   || \mathbf{z}^{(1)}_i- \mathbf{z}^{(2)}_i||_2^2.
        \label{eq2}
    \end{equation}
    Based on Eq. (\ref{apeq1}), we can rewrite Eq. (\ref{eq2}) as:
\begin{equation}
\begin{aligned}
&~~~~~\min_{\theta } \sum_{i}^{n}   || \mathbf{z}^{(1)}_i- \mathbf{z}^{(2)}_i||_2^2 \\
&=\min_{\theta } \sum_{i}^{n}   || \mathbf{h}_i + \{\mathbf{h}_j, v_j\in \mathcal{N}(v_i)^{(1)}\}- \mathbf{h}_i - \{\mathbf{h}_k, v_k\in \mathcal{N}(v_i)^{(2)}\}||_2^2
\\
&=\min_{\theta } \sum_{i}^{n}   || \mathbf{h}_i  - \{\mathbf{h}_k, v_k\in \mathcal{N}(v_i)^{(2)}\} + \{\mathbf{h}_j, v_j\in \mathcal{N}(v_i)^{(1)}\}- \mathbf{h}_i||_2^2\\
&=  \min_{\theta } \sum_{i}^{n}   || \mathbf{h}_i  - \{\mathbf{h}_k, v_k\in \mathcal{N}(v_i)^{(2)}\}||_2^2 + ||\{\mathbf{h}_j, v_j\in \mathcal{N}(v_i)^{(1)}\}- \mathbf{h}_i||_2^2 \\  
&~~~~ +2\sum_{i}^{n} \langle  (\mathbf{h}_i  - \{\mathbf{h}_k, v_k\in \mathcal{N}(v_i)^{(2)}\})\cdot(\{\mathbf{h}_j, v_j\in \mathcal{N}(v_i)^{(1)}\}- \mathbf{h}_i) \rangle \\
&= \min_{\theta } \sum_{i}^{n}\sum_{k}^{n}  \mathcal{G}^{(2)}_{i,k} ||\mathbf{h}_i-\mathbf{h}_k||_2^2 + \sum_{i}^{n}\sum_{j}^{n}  \mathcal{G}^{(1)}_{i,j} ||\mathbf{h}_i-\mathbf{h}_j||_2^2 \\
&~~~~ +2\sum_{i}^{n} \langle  (\mathbf{h}_i  - \{\mathbf{h}_k, v_k\in \mathcal{N}(v_i)^{(2)}\})\cdot(\{\mathbf{h}_j, v_j\in \mathcal{N}(v_i)^{(1)}\}- \mathbf{h}_i) \rangle \\
&= \min_{\theta } \sum_{i}^{n}\sum_{l}^{n} \mathcal{G}_{i,l} ||\mathbf{h}_i-\mathbf{h}_l||_2^2 +2\sum_{i}^{n} \langle  (\{\mathbf{h}_i  - \{\mathbf{h}_k, v_k\in \mathcal{N}(v_i)^{(2)}) \\ 
&~~~~\cdot(\{\mathbf{h}_j, v_j\in \mathcal{N}(v_i)^{(1)}\}- \mathbf{h}_i) \rangle .
\end{aligned}  
\label{eq3}
\end{equation}
Denote $\mathbf{D}$ as the degree matrix of $\mathcal{G}$, denote $\mathbf{L}=\mathbf{D}-\mathcal{G}$ as the graph laplacian, and  denote $\mathbf{h}^i \in \mathbb{R}^{n}$ $\mathbf{h}_j \in \mathbb{R}^{d}$ is the $i$-th column and $j$-th row of $\mathbf{H}$, according to the spectral graph analysis in \cite{von2007tutorial}, we further have
\begin{equation}
    \begin{aligned}
(\mathbf{h}^i )^T\mathbf{L}(\mathbf{h}^i )&=(\mathbf{h}^i )^T\mathbf{D}(\mathbf{h}^i )^T-(\mathbf{h}^i )^T\mathcal{G}(\mathbf{h}^i )^T\\
&=\sum_{j=1}^{n} \mathbf{d}_{jj} (\mathbf{h}^{ij})^{2}-\sum_{j, k=1}^{n} \mathbf{h}^{ij} \mathbf{h}^{ik} \mathcal{G}_{j k}\\
&=\frac{1}{2} (\sum_{j=1}^{n} \mathbf{d}_{jj} (\mathbf{h}^{ij})^{2}-2\sum_{j, k=1}^{n}\mathbf{h}^{ij} \mathbf{h}^{ik} \mathcal{G}_{j k}+\sum_{k=1}^{n} \mathbf{d}_{kk} (\mathbf{h}^{ik})^{2})\\
&=\frac{1}{2} \sum_{j, k=1}^{n}\mathcal{G}_{j k}(\mathbf{h}^{ij} - \mathbf{h}^{ik})^2.
\end{aligned}
\label{eq77}
\end{equation}
Therefore, we further have 
\begin{equation}
    \begin{aligned}
\sum_{i=1}^{d} (\mathbf{h}^i )^T\mathbf{L}(\mathbf{h}^i )=\frac{1}{2} \sum_{j, k=1}^{n}\mathcal{G}_{j k}\sum_{i=1}^{d}(\mathbf{h}^{ij} - \mathbf{h}^{ik})^2.
\end{aligned}
\label{eq88}
\end{equation}
That is
\begin{equation}
    \text{Tr}(\mathbf{H}^T\mathbf{L}\mathbf{H}) = \frac{1}{2} \sum_{j, k=1}^{n}\mathcal{G}_{j k}||\mathbf{h}_{j} - \mathbf{h}_{k}||^2_2.
    \label{eq99}
\end{equation}

where $\text{Tr}(\cdot)$ indicates the matrix trace.
Therefore, based on Eq. (\ref{eq3}) and Eq. (\ref{eq99}), we can obtain
\begin{equation}
    \begin{aligned}
&~~~~~\min_{\theta } \sum_{i}^{n}   || \mathbf{z}^{(1)}_i- \mathbf{z}^{(2)}_i||_2^2 \\
&= \min_{\theta }2\text{Tr}(\mathbf{H}^T\mathbf{L}\mathbf{H})+2\sum_{i}^{n} \langle  (\mathbf{h}_i  - \{\mathbf{h}_k, v_k\in \mathcal{N}(v_i)^{(2)}\})\cdot(\mathbf{h}_j, v_j\in \mathcal{N}(v_i)^{(1)}\}- \mathbf{h}_i\}) \rangle\\
&=  \min_{\theta }2\text{Tr}(\mathbf{H}^T\mathbf{L}\mathbf{H})+2\sum_{i,j,k}^{n}\mathcal{G}_{i,j}\mathcal{G}_{i,k}\langle  (\mathbf{h}_i  -\mathbf{h}_j)\cdot (\mathbf{h}_i  -\mathbf{h}_k)\rangle.
\end{aligned}
\end{equation}
Based on the assumption that $\mathbf{H}^T\mathbf{H} = \mathbf{I}$, we can conclude that previous meta-path-based SHGL methods, which extract the invariance among different graphs, equals the known spectral clustering with additional regularization. Note that the MSE loss in the above example can be replaced by other contrastive or non-contrastive loss (\eg InfoNCE \cite{oord2018representation}), and we can easily obtain similar results.

After that, we further prove the connection between recent adaptive-graph-based SHGL methods \cite{mo2024selfsupervised} and the spectral clustering. Denote the self-expressive matrix in \cite{mo2024selfsupervised} as $\mathbf{S}$, and denote the representations after projection by linear transformation as $\mathbf{H}$. Moreover, denote $\mathbf{D}_\mathbf{S}$ as the degree matrix of $\mathbf{S}$ and denote $\mathbf{L}_\mathbf{S}=\mathbf{D}_\mathbf{S}-\mathbf{S}$ as the graph Laplacian. Given that the self-expressive matrix is symmetrical and non-negative, the objective function of previous adaptive-graph-based SHGL methods can be formulated as:
\begin{equation}
\begin{aligned}
    &~~~~~\min_{\theta,\mathbf{S}}\| \mathbf{H}-\mathbf{S}\mathbf{H}\|_F^2+\alpha\sum_{i,j=1}^{n}d_{ij}\mathbf{s}_{ij}+\beta\sum_{i,j=1}^{n}\mathbf{s}_{ij}^2 ,
      \label{apeq6}
\end{aligned}
\end{equation}
where $\alpha$ and $\beta$ are non-negative parameters, and $d_{ij}$ indicates the distance among nodes based on $\mathbf{H}$ or original node features.
Based on the self-expressive constraint in the first term of Eq. (\ref{apeq6}), we have  
\begin{equation}
    \mathbf{h}_{i}=\sum_{j=1, j \neq i}^{n} \mathbf{s}_{i j} \mathbf{h}_{j}, \forall 1\le i\le n.
    \label{apeq7}
\end{equation}
Therefore, for any $\mathbf{h}_i$ where  $i \in [1, n]$, we further have
\begin{equation}
    \mathbf{h}_{i}^T\mathbf{h}_{i}=\sum_{j=1, j \neq i}^{n} \mathbf{s}_{i j} \mathbf{h}_{i}^T\mathbf{h}_{j}.
\end{equation}
Based on the constraint $\mathbf{s}_i^T\mathbf{1}=1$, we obtain
\begin{equation}
    (\sum_{j=1}^{n} \mathbf{s}_{ij}+\sum_{j=1}^{n} \mathbf{s}_{i j}) \mathbf{h}_{i}^{T} \mathbf{h}_{i}=2 \sum_{j=1, i \neq j} \mathbf{s}_{i j} \mathbf{h}_{i}^{T} \mathbf{h}_{j} .
\end{equation}
Therefore, the constraint in Eq. (\ref{apeq7}) can be transformed as:
\begin{equation}
    (\sum_{j=1}^{n} \mathbf{s}_{ij}+\sum_{j=1}^{n} \mathbf{s}_{i j}) \mathbf{h}_{i}^{T} \mathbf{h}_{i}-2 \sum_{j=1, i \neq j} \mathbf{s}_{i j} \mathbf{h}_{i}^{T} \mathbf{h}_{j}=0 .
\end{equation}
In addition, we further have
\begin{equation}
    \sum_{i=1}^{n}((\sum_{j=1}^{n} \mathbf{s}_{ij}+\sum_{j=1}^{n} \mathbf{s}_{i j}) \mathbf{h}_{i}^{T} \mathbf{h}_{i}-2 \sum_{j=1, i \neq j} \mathbf{s}_{i j} \mathbf{h}_{i}^{T} \mathbf{h}_{j}) =\sum_{i=1}^{n} \sum_{j=1}^{n}\left\|\mathbf{h}_{i}-\mathbf{h}_{j}\right\|^{2} \mathbf{s}_{i j} .
\end{equation}

Similar to the proof above, we can also rewrite Eq. (\ref{apeq6}) as:
\begin{equation}
\begin{aligned}
    &~~~~~\min_{\theta,\mathbf{S}}2\text{Tr}(\mathbf{H}^T\mathbf{L}\mathbf{H})+\alpha\sum_{i,j=1}^{n}d_{ij}\mathbf{s}_{ij}+\beta\sum_{i,j=1}^{n}\mathbf{s}_{ij}^2 .
      \label{apeq12}
\end{aligned}
\end{equation}

Based on the assumption that $\mathbf{H}^T\mathbf{H} = \mathbf{I}$, we can also conclude that previous adaptive-graph-based SHGL methods are equal to the known spectral clustering with additional regularization. Therefore, we complete the proof.

\end{proof}

\subsection{Proof of Theorem \ref{thm2}}
\label{proof23}
To prove Theorem \ref{thm2}, we first give the definition of the graph-cut algorithm as follows.
\begin{definition} (Graph-Cut)
For a given number k of subsets, the
mincut approach simply consists in choosing a partition $V_1, ...,V_d$ which minimizes
\begin{equation}
\begin{aligned}
&\operatorname{Cut}\left(V_{1}, \ldots, V_{d}\right):=\frac{1}{2} \sum_{i=1}^{d} \mathbf{W}\left(V_{i}, \bar{V}_{i}\right),\\
&\operatorname{RatioCut}\left(V_{1}, \ldots, V_{d}\right):=\frac{1}{2} \sum_{i=1}^{d}  \frac{\mathbf{W}\left(V_{i}, \bar{V}_{i}\right)}{\left|V_{i}\right|}= \sum_{i=1}^{d}  \frac{\operatorname{Cut}\left(V_{i}, \bar{V}_{i}\right)}{\left|V_{i}\right|},
\end{aligned}
\end{equation}
where $\mathbf{W}(V_a, V_b):=\sum_{i \in V_a, j \in V_b} w_{i j}$ indicates the weight between different subsets, and $\bar{V}$ is the complement of $V$. 

\end{definition}

\begin{theorem}
{\rm(Restating Theorem \ref{thm2} in the main
text).}
Under the same assumption in Theorem \ref{thm1}, optimizing previous meta-path-based  and adaptive-graph-based SHGL methods is approximate to performing the $\operatorname{RatioCut}\left(V_{1}, \ldots, V_{d}\right)$  algorithm that divides the learned representations into $d$ partitions $\{V_{1}, \ldots, V_{d}\}$, i.e.,
\begin{equation}
        \min _{\mathbf{H}}\mathcal{L}_{SHGL}\cong \min _{\mathbf{H}} \operatorname{RatioCut}\left(V_{1}, \ldots, V_{d}\right),
\end{equation}
where $d$ indicates the dimension of representations $\mathbf{H}$.
\label{apthm2}
\end{theorem}

\begin{proof}
Given a partition of $V$ with $n$ samples into $d$ sets $V_1, ...,V_d$, we first define $d$ indicator vectors
$\mathbf{h}_j = (\mathbf{h}_{1,j}, ..., \mathbf{h}_{n,j})'$ by
\begin{equation}
    \mathbf{h}_{i, j}=\left\{\begin{array}{ll}
1 / \sqrt{\left|{V}_{j}\right|} & \text { if } v_{i} \in {V}_{j} \\
0 & \text { otherwise }
\end{array} \quad(i=1, \ldots, n ; j=1, \ldots, d)\right.
\label{eq14}
\end{equation}
Then we set the matrix $\mathbf{H} \in \mathbb{R}^{n \times d}$ as the matrix containing those $d$ indicator vectors as columns. Observe that the columns in $\mathbf{H}$ are orthonormal to each other, \ie $\mathbf{H}^T\mathbf{H} = \mathbf{I}$, where $\mathbf{I}$ is the identity matrix. Denote $\mathbf{L}$ as the unnormalized graph Laplacian, according to \cite{von2007tutorial}, we can obtain
\begin{equation}
    \begin{aligned}
\mathbf{h}_i^{T} \mathbf{L} \mathbf{h}_i & =\frac{1}{2}   \sum_{j, k=1}^{|V_i \cup \bar{V}_i|} w_{j k}(\mathbf{h}_{j}-\mathbf{h}_{k})^{2} \\
& =\frac{1}{2} \sum_{j \in {V}_i, k \in \bar{{V_i}}} w_{j k}\left(\sqrt{\frac{|\bar{V_i}|}{|V_i|}}+\sqrt{\frac{|V_i|}{|\bar{V_i}|}}\right)^{2}+\frac{1}{2} \sum_{j \in \bar{V_i}, k \in V_i} w_{jk}\left(-\sqrt{\frac{|\bar{V_i}|}{|V_i|}}-\sqrt{\frac{|V_i|}{|\bar{V_i}|}}\right)^{2} \\
& =  \operatorname{cut}(V_i, \bar{V_i})\left(\frac{|\bar{V_i}|}{|V_i|}+\frac{|V_i|}{|\bar{V_i}|}+2\right) \\
& = \operatorname{cut}(V_i, \bar{V_i})\left(\frac{|V_i|+|\bar{V_i}|}{|V_i|}+\frac{|V_i|+|\bar{V_i}|}{|\bar{V_i}|}\right) \\
& =  |V_i| \cdot \operatorname{RatioCut}(V_i, \bar{V_i}) .
\end{aligned}
\end{equation}

Moreover, we have $\mathbf{h}_i^{T} \mathbf{L}\mathbf{h}_i = (\mathbf{H}^T\mathbf{L}\mathbf{H})_{ii}$. Therefore, we have
\begin{equation}
\operatorname{Tr}(\mathbf{H}^T\mathbf{L}\mathbf{H}) = \sum_{i =1}^{d}(\mathbf{H}^T\mathbf{L}\mathbf{H})_{ii}=\sum_{i =1}^{d}\mathbf{h}_i^{T} \mathbf{L} \mathbf{h}_i=\operatorname{RatioCut}\left(V_{1}, \ldots, V_{d}\right),
\end{equation}
where $\operatorname{Tr}(\cdot)$ indicates the trace of a matrix. Therefore, minimizing the $\operatorname{RatioCut}\left(V_{1}, \ldots, V_{k}\right)$ can be transferred to 
\begin{equation}
    \min_{V_{1}, \ldots, V_{d}}\operatorname{Tr}(\mathbf{H}^T\mathbf{L}\mathbf{H}) ~~ \operatorname{s.t.,} \mathbf{H}^T\mathbf{H}=\mathbf{I}, \mathbf{H} \operatorname{as ~~ defined~~ in ~~Eq. (\ref{eq14})}.
\end{equation}
 Then we consider relaxing the constraints of the problem by allowing the entries of the matrix $\mathbf{H}$ to assume arbitrary real values. As a result, the problem is transformed into a relaxed version:
\begin{equation}
    \min_{\mathbf{H}\in \mathbb{R}^{n \times d}}\operatorname{Tr}(\mathbf{H}^T\mathbf{L}\mathbf{H}) ~~ \operatorname{s.t.,} \mathbf{H}^T\mathbf{H}=\mathbf{I}.
\end{equation}
This is the standard spectral clustering, as we mentioned above. Therefore, we can obtain that conducting spectral clustering is approximating to conducting the RatioCut algorithm. which divide the learned representations into $d$ partitions, where $d$ indicates the dimension of representations. Thus, we complete the proof.

\end{proof}

\subsection{Proof of Theorem \ref{thm3}}
\label{proof25}
\begin{theorem}
{\rm(Restating Theorem \ref{thm3} in the main
text).}
    Optimizing the spectral loss $\mathcal{L}_{sp}$ leads to performing the spectral clustering  based on the affinity matrix $\mathbf{S}$ with $c$ connected components and conducting RatioCut ($V_1,\ldots, V_{c}$) algorithm to divide the learned representations into $c$ partitions, i.e.,
\begin{equation}
\min\mathcal{L}_{sp}\Rightarrow\min\operatorname{Tr} (\mathbf{Y}^T\mathbf{L}_\mathbf{S}\mathbf{Y})\Rightarrow\min \operatorname{RatioCut}(V_{1},\ldots,V_{c}).
\end{equation}
    \label{apthm3}
\end{theorem}

\begin{proof}
    According to Ky Fan’s Theorem \cite{fan1949theorem}, the spectral loss $\mathcal{L}_{sp}$ can be written as:
    \begin{equation}
\begin{aligned}
\mathcal{L}_{sp}&=\frac{1}{n^{2}} \sum_{i, j=1}^{n} \mathbf{s}_{i j}\left\|{\mathbf{y}}_{i}-{\mathbf{y}}_{j}\right\|^{2} - \gamma H(\mathbf{Y})\\
&=\frac{2}{n^{2}} \operatorname{Tr}(\mathbf{Y}^T\mathbf{L}_\mathbf{S}\mathbf{Y}) -\gamma H(\mathbf{Y}).
\end{aligned}    
\end{equation}

Therefore, optimizing the proposed method with $\mathcal{L}_{sp}$  is equivalent to performing the spectral clustering with additional regularization. 

Note that, under the orthogonal constraint, the minimum of $\mathcal{L}_{sp}$ is attained when the column space of $\mathbf{Y}$ is the subspace of the $c$ eigenvectors corresponding to the smallest $c$ eigenvalues
of $\mathbf{L}_\mathbf{S}$. In other words, the learned $\mathbf{Y}$ can perfectly fit the eigenvectors when the minimum of $\mathcal{L}_{sp}$ is attained. Recall the objective function in the main text, \ie 
\begin{equation}
        \begin{array}{l}
\min _{\mathbf{S},\mathbf{F}} \sum_{i, j=1}^{n}(\|\mathbf{h}_{i}-\mathbf{h}_{j}\|_{2}^{2} \mathbf{s}_{i j}+\alpha\mathbf{s}_{i j}^{2}+\beta\left\|\mathbf{f}_{i}-\mathbf{f}_{j}\right\|_{2}^{2}\mathbf{s}_{i j}) \\
\text { s.t., } \forall i, \mathbf{s}_{i}^{T} \mathbf{1}=1,0 \leq \mathbf{s}_{i} \leq 1,\mathbf{F}^T\mathbf{F}=\mathbf{I}.
\end{array}
\end{equation}
Therefore, when the minimum of $\mathcal{L}_{sp}$ is attained, the constraints in the above function can be satisfied, \ie $\operatorname{rank}(\mathbf{L}_\mathbf{S}) = n-c$ holds. As a result, we can obtain the  affinity matrix $\mathbf{S}$ with exactly $c$ connected components.

Moreover, according to the Theorem \ref{apthm2}, 
we have
\begin{equation}
\operatorname{Tr}(\mathbf{Y}^T\mathbf{L}\mathbf{Y}) = \sum_{i =1}^{c}(\mathbf{Y}^T\mathbf{L}_\mathbf{S}\mathbf{Y})_{ii}=\sum_{i =1}^{c}\mathbf{y}_i^{T} \mathbf{L}_\mathbf{S}\mathbf{y}_i=\operatorname{RatioCut}\left(V_{1}, \ldots, V_{c}\right).
\end{equation}
 That is, the proposed method divides the learned representations into $c$ partitions, where $c$ indicates the number of classes. Thus, we complete the proof.

\end{proof}

\subsection{Proof of Theorem \ref{thm4}}
We first follow previous works \cite{natekar2020representation} to define the Complexity Measure to evaluate the generalization ability of neural networks based on the Davies Bouldin Index.
\begin{definition} (Complexity Measure)
    The complexity measure of neural networks can be defined as:
    \begin{equation}
        C=\frac{1}{k} \sum_{i=0}^{k-1} \max _{i \neq j} \frac{S_{i}+S_{j}}{M_{i, j}},
        \label{eq20}
    \end{equation}
where 
\begin{equation}
\begin{array}{c}
S_{i}=\left(\frac{1}{n_{i}} \sum_{\tau}^{n_{i}}\left|O^{i}_{\tau}-\mu_{{i}}\right|^{p}\right)^{1 / 2} \text { for } i=1 \cdots k \\
M_{i, j}=\left\|\mu_{{i}}-\mu_{{j}}\right\|_{2} \quad \text { for } i, j=1 \cdots k,
\end{array}
\end{equation}
$i$ and $j$  are indices of two different classes, $O_{i}^{(\tau)}$ is the output representation of the $\tau$-th sample belonging to class $i$ for the given model, $\mu_i$ is the cluster centroid of the representations of class $i$, $S_{i}$ is a measure of scatter within representations of class $i$, and $M_{i, j}$ is a
measure of separation between representations of classes $i$ and $j$.
\end{definition}

Moreover, we further follow previous works \cite{natekar2020representation, jiang2021methods} to define the generalization bound $G$ of a model based on the model complexity, \ie
\begin{definition} (Generalization Bound)
For any $\delta \in [0, 1]$, with probability at least $1-\delta$,  the generalization bound $G$ of a model follows the inequality, \ie 
\begin{equation}
G \le \frac{1}{n} \sum_{i=1}^{n} \ell(f(\mathbf{x}_{i}), \mathbf{y}_{i})+\sqrt{\frac{C}{n}}+\mathcal{O}(\sqrt{\frac{\log (1 / \delta)}{n}}),
\end{equation}
where $(\mathbf{x}_{i}, \mathbf{y}_{i})$ is a pair of labeled data, $f$ is the model, $l$ is the loss function, $n$ is the number of labeled data, $C$ is the model complexity measure.
\end{definition}
Based on the Definitions above, we can derive the Theorem as follows.
\label{proof26}
\begin{theorem}
{\rm(Restating Theorem \ref{thm4} in the main
text).}
The proposed method with dual  consistency constraints achieves a lower boundary of the model complexity $C$ and a higher generalization ability boundary $G$ than previous SHGL with the node-level consistency constraint only, i.e.,
\begin{equation}
\inf(C_{SCHOOL}) < \inf(C_{SHGL}),\quad
\sup(G_{SCHOOL}) > \sup(G_{SHGL}),
\end{equation}
    where $\inf(\cdot)$ and $\sup(\cdot)$ indicates lower bound and upper bound, respectively.
    \label{apthm4}
\end{theorem}
\begin{proof}
    We take the binary classification as an example, Eq. (\ref{eq20}) can be rewritten as $\frac{S_{0}+S_{1}}{M_{0,1}}$. Then, for the heterogeneous representations $\Tilde{\mathbf{Z}}$ learned by the heterogeneous encoder, we can obtain its cluster centroid $\mu_0$:
    \begin{equation}
        \begin{aligned}
\mu_{{0}} & =\mathbb{E}[\tilde{\mathbf{z}}_{i}^{0}]=\mathbb{E}[\mathbf{W}(\mathbf{X}_i +\sum_{j \in \mathcal{N}\left(v_{i}\right)} \frac{1}{d} \mathbf{X}_{j})] \\
& =\mathbf{W}\left(P_{0} \cdot \mu_{\mathbf{X}_{0}}+\left(1-P_{0}\right) \cdot \mu_{\mathbf{X}_{j}}\right),
\end{aligned}
    \end{equation}
    where $\mathcal{N}(v_{i})$ indicates the neighbors of $v_i$ from other type of nodes, $\mathbf{W}$ indicates the parameters of the heterogeneous encoder, $\mu_{\mathbf{X}_{0}}$ indicates the cluster centroid of the node features of class $i$, $\mu_{\mathbf{X}_{j}}$ indicates the cluster centroid of the node features of other types of nodes.
Similarly, we further have:
\begin{equation}
    \mu_{{1}} = \mathbf{W}\left(P_{1} \cdot \mu_{\mathbf{X}_{1}}+\left(1-P_{1}\right) \cdot \mu_{\mathbf{X}_{k}}\right),
\end{equation}
where $\mu_{\mathbf{X}_{k}}$ indicates the cluster centroid of the node features of other types of nodes.

Therefore, we can obtain 
\begin{equation}
    \begin{aligned}
M_{0,1} & =\left\|\mu_{{0}}-\mu_{{1}}\right\| \\
& =\left\|\mathbf{W}\left(P_{0} \cdot \mu_{\mathbf{X}_{0}}+\left(1-P_{0}\right) \cdot \mu_{\mathbf{X}_{j}}-P_{1} \cdot \mu_{\mathbf{X}_{1}}-\left(1-P_{1}\right) \cdot \mu_{\mathbf{X}_{k}}\right)\right\|. 
\end{aligned}
\end{equation}
Moreover, we have
\begin{equation}
    \begin{aligned}
S_{0}^{2} & =\mathbb{E}[\left\|O_{i}^{0}-\mu_{{0}}\right\|^{2}]=\mathbb{E}\left[<O_{\tau}^{0}-\mu_{{0}}, O_{\tau}^{0}-\mu_{{0}}>\right] \\
& =P_{0}^{2} \mathbb{E}[\left\|\mathbf{W}\left(\mathbf{X}_{i}^{0}-\mu_{\mathbf{X}_{0}}\right)\right\|^{2}]+\left(1-P_{0}\right)^{2} \mathbb{E}[\|\mathbf{W}(\mathbf{X}_{i}^{j}-\mu_{\mathbf{X}_{j}})\|^{2}],
\end{aligned}
\label{eq52}
\end{equation}
where $<\cdot,\cdot > $ is inner production. To rewrite the above function, we first derive the following inequality, \ie
\begin{equation}
a^2b+(1-a)^2c\ge\frac{bc}{b+c}, 
\label{eq53}
\end{equation}
where $0\le a\le1$, and $0\le b,c$. To prove the inequaility in Eq. (\ref{eq53}), we construct function $f(a) = a^2b+(1-a)^2c-\frac{bc}{b+c}$.
We then take the derivative of $f(a)$, i.e.,
\begin{equation}
f'(a)=2ab-2c+2ac.
\end{equation}
Then we let $f'(a)=0$, and have $a=\frac{c}{b+c}$.
We have 
\begin{equation}
f(\frac{c}{b+c})=\frac{c^2b}{(b+c)^2}+\frac{b^2c}{(b+c)^2}-\frac{bc}{b+c}=0.
\end{equation}
In addition, we take the second-order derivative of $f(a)$ and obtain
\begin{equation}
f''(a)=2(b+c)\ge0.
\end{equation}
Therefore, $f(a)$ is decreasing when $a<\frac{c}{b+c}$ and increasing  when $a>\frac{c}{b+c}$, and reaches its minimum 0 at $\frac{c}{b+c}$. As a result, $f(a)\ge0$ always holds for $0\le a\le1$. Thus, we prove the inequality in Eq. (\ref{eq53}).

Given the above inequality in Eq. (\ref{eq53}), for Eq. (\ref{eq52}), we let $\sigma_{0}^{2}=\mathbb{E}[\|\mathbf{W}(\mathbf{X}_{0}^{(i)}-\mu_{\mathbf{X}_{0}})\|^{2}]$, $\sigma_{1}^{2}=\mathbb{E}[\|\mathbf{W}(\mathbf{X}_{1}^{(i)}-\mu_{\mathbf{X}_{1}})\|^{2}]$, $\sigma_{j}^{2}=\mathbb{E}[\|\mathbf{W}(\mathbf{X}_{j}^{(i)}-\mu_{\mathbf{X}_{j}})\|^{2}]$, and $\sigma_{k}^{2}=\mathbb{E}[\|\mathbf{W}(\mathbf{X}_{k}^{(i)}-\mu_{\mathbf{X}_{k}})\|^{2}]$. Moreover, we replace $a$, $b$, and $c$ in Eq. (\ref{eq53}) with $P_0$, $\sigma_0^2$, and $\sigma_j^2$, respectively. 
Then Eq. (\ref{eq52}) can be rewritten as: 
\begin{equation}
    S_{0}^{2}=P_{0}^{2} \sigma_{0}^{2}+\left(1-P_{0}\right)^{2} \sigma_{j}^{2} \geq \frac{\sigma_{0}^{2} \sigma_{j}^{2}}{\sigma_{0}^{2}+\sigma_{j}^{2}}.
    \label{eq26}
\end{equation}
Similarly, we can also reach the following inequality:  
\begin{equation}
    S_{1}^{2}=P_{1}^{2} \sigma_{1}^{2}+\left(1-P_{1}\right)^{2} \sigma_{k}^{2} \geq \frac{\sigma_{1}^{2} \sigma_{k}^{2}}{\sigma_{1}^{2}+\sigma_{k}^{2}}.
    \label{eq27}
\end{equation}
Therefore, the  complexity measure $C$ can calculated by:
\begin{equation}
    C=\frac{\sqrt{S_{0}^{2}}+\sqrt{S_{1}^{2}}}{M_{0,1}} \geq \frac{\frac{\sigma_{0} \sigma_{j}}{\sqrt{\sigma_{0}^{2}+\sigma_{j}^{2}}} + \frac{\sigma_{1} \sigma_{k}}{\sqrt{\sigma_{1}^{2}+\sigma_{k}^{2}}}}{ M_{0,1}}.
\end{equation}
Note that the cluster-level consistency constraint minimizes the first term in the 
$S_{0}^{2}$ and $S_{1}^{2}$ (\ie $\sigma_{0}^2$ and $\sigma_{1}^2$). Moreover, we can observe that Eq. (\ref{eq26}) and Eq. (\ref{eq27}) are the increasing function with respect to $\sigma_{0}$ and $\sigma_{1}$. Therefore, minimizing the cluster-level consistency constraint is equivalent to minimizing the lower bound of the model complexity. Therefore, the lower bound of complexity measure $C$ of the model with the dual consistent constraints is less than the model without it, \ie $\inf(C_{SCHOOL}) < \inf(C_{SHGL})$. As a result, according to \cite{natekar2020representation}, we can conclude that the representations learned by the dual consistent constraints have a higher bound of generalization ability than previous methods with instance-level constraint only, \ie $\sup(G_{SCHOOL}) > \sup(G_{SHGL})$ thus we complete the proof. 
\end{proof}

\subsection{Derivation of Eq. (\ref{eq7}).}
\label{derivationeq7}
Recalling Eq. (\ref{eq7}), \ie
\begin{equation}
     \sum_{i=1}^{c} \tau _{i}\left(\mathbf{L}_\mathbf{S}\right) 
    = \min _{\mathbf{F}^{T} \mathbf{F}=\mathbf{I}} \operatorname{Tr}(\mathbf{F}^{T} \mathbf{L}_\mathbf{S} \mathbf{F})
    = \min _{\mathbf{F}^{T} \mathbf{F}=\mathbf{I}} \frac{1}{2} \sum_{i j} \mathbf{s}_{i j}\left\|\mathbf{f}_{i}-\mathbf{f}_{j}\right\|_{2}^{2},
    \label{eq36}
\end{equation}
where $\mathbf{F} \in \mathbb{R}^{n \times c}$ is the eigenvector (\ie $\mathbf{F}^{T} \mathbf{F}=\mathbf{I}$) of $\mathbf{L}_\mathbf{S}$ corresponding to the $c$ eigenvalues.
We first derive the first equation. The eigendecomposition of the symmetric  $\mathbf{L}_\mathbf{S}$ can be written as: $\mathbf{L}_\mathbf{S} = \mathbf{B}\Lambda\mathbf{B}^T$, where $\mathbf{B}$ is the eigenvector matrix and $\Lambda$ is the diagonal matrix whose diagonal elements are the eigenvalues of  $\mathbf{L}_\mathbf{S}$. 
We have:
\begin{equation}
\begin{array}{l}
\operatorname{Tr}(\mathbf{F}^{T} \mathbf{L}_\mathbf{S} \mathbf{F})=
\operatorname{Tr}(\left(\begin{array}{c}
\mathbf{f}_{1}^{T} \\
\mathbf{f}_{2}^{T} \\
\vdots \\
\mathbf{f}_{c}^{T}
\end{array}\right) \mathbf{L}_\mathbf{S}\left(\begin{array}{llll}
\mathbf{f}_{1} & \mathbf{f}_{2} & \cdots & \mathbf{f}_{c})
\end{array}\right) \\
=\operatorname{Tr}(\left(\begin{array}{cccc}
\mathbf{f}_{1}^{T} \mathbf{L}_\mathbf{S} \mathbf{f}_{1} & \mathbf{f}_{1}^{T} \mathbf{L}_\mathbf{S} \mathbf{f}_{2} & \cdots & \mathbf{f}_{1}^{T} \mathbf{L}_\mathbf{S} \mathbf{f}_{c} \\
\mathbf{f}_{2}^{T} \mathbf{L}_\mathbf{S} \mathbf{f}_{1} & \mathbf{f}_{2}^{T} \mathbf{L}_\mathbf{S} \mathbf{f}_{2} & \cdots & \mathbf{f}_{2}^{T} \mathbf{L}_\mathbf{S} \mathbf{f}_{c} \\
\vdots & \vdots & \ddots & \vdots \\
\mathbf{f}_{c}^{T}\mathbf{L}_\mathbf{S} \mathbf{f}_{1} & \mathbf{f}_{c}^{T} \mathbf{L}_\mathbf{S} \mathbf{f}_{2} & \cdots & \mathbf{f}_{c}^{T} \mathbf{L}_\mathbf{S} \mathbf{f}_{c}
\end{array}\right)) \\
=\sum_{i=1}^{c}\Lambda_{i},
\end{array}
\end{equation}
where $\sum_{i=1}^{c}\Lambda_{i}$ indicates the sum of any $c$ eigenvalues of $\mathbf{L}_\mathbf{S}$. Obviously, $\min _{\mathbf{F}^{T} \mathbf{F}=\mathbf{I}} \operatorname{Tr}(\mathbf{F}^{T} \mathbf{L}_\mathbf{S} \mathbf{F})$ achieves its minimization when $\mathbf{F}$ is the eigenvectors corresponding to $c$ smallest eigenvalues.
Therefore we have $\min_\mathbf{S} \sum_{i=1}^{c} \tau _{i}\left(\mathbf{L}_\mathbf{S}\right) = \min _{\mathbf{F}^{T} \mathbf{F}=\mathbf{I}} \operatorname{Tr}(\mathbf{F}^{T} \mathbf{L}_\mathbf{S} \mathbf{F})$ and the first equation in Eq. (\ref{eq36}) is proved. Moreover, based on Eq. (\ref{eq77})-Eq. (\ref{eq99}), we can further complete the proof the second equation in Eq. (\ref{eq36}).

\subsection{Derivation of the Closed-Form Solution and Parameters.}
\label{closedform}
We first obtain the Lagrangian function of the objective function in Eq. (9) in the main text:
\begin{equation}
    \mathcal{L}(\mathbf{s}_i,\lambda,\varepsilon )= \|\mathbf{s}_{i}+\frac{1}{2 {\alpha}} \mathbf{d}_{i}\|_{2}^{2}-\lambda (\mathbf{s}_i^T-1)-\varepsilon_i^T\mathbf{s}_i,
\end{equation}
where $\lambda$ and $\varepsilon_i \ge \mathbf{0}$ are the Lagrangian multipliers. Based on the KKT condition \cite{boyd2006convex}, we can obtain the closed-form solution of the above Lagrangian function, \ie 
\begin{equation}
    \mathbf{s}_{ij} = (-\frac{1}{2 {\alpha_i}} \mathbf{d}_{ij} + \lambda_i )_+,
    \label{eq37}
\end{equation}
where $(\cdot)_+$ indicates $\max \{\cdot, 0\}$. For the sparse affinity matrix $\mathbf{S}$, each vector $\mathbf{s}_i$ contains $k$ nonzero elements only. Therefore, we have $\mathbf{s}_{ik} \ge 0$ and  $\mathbf{s}_{i,k+1} = 0$. That is, $-\frac{1}{2 {\alpha_i}} \mathbf{d}_{ik} + \lambda_i > 0$ and $-\frac{1}{2 {\alpha_i}} \mathbf{d}_{i,k+1} + \lambda_i \le 0$. Then based on Eq. (\ref{eq37}) and the constraint $\mathbf{s}_i^T\mathbf{1} = 1$, we further have  
\begin{equation}
    \sum_{j=1}^{k}( -\frac{1}{2 {\alpha_i}} \mathbf{d}_{ij} + \lambda_i)=1.
\end{equation}
Therefore, we obtain $\lambda_i=\frac{1}{k}+\frac{1}{2 k \alpha_{i}} \sum_{j=1}^{k} \mathbf{d}_{i j}$.
Moreover, we have the following inequality for $\alpha_i$, \ie
\begin{equation}
    \frac{k}{2} \mathbf{d}_{i k}-\frac{1}{2} \sum_{j=1}^{k} \mathbf{d}_{i j}<\alpha_{i} \leq \frac{k}{2} \mathbf{d}_{i, k+1}-\frac{1}{2} \sum_{j=1}^{k} \mathbf{d}_{i j}.
\end{equation}
Hence, to achieve an optimal solution $\mathbf{s}_i$ contain precisely $k$ non-zero values, we can set $\alpha_i$ as:
\begin{equation}
    \alpha_i=\frac{k}{2} \mathbf{d}_{i, k+1}-\frac{1}{2} \sum_{j=1}^{k} \mathbf{d}_{i j}.
\end{equation}
Then the overall $\alpha$ could be set to the mean of $\alpha_1, \alpha_2, ..., \alpha_n$. Then $\alpha$ can be obtained by:
\begin{equation}
    \alpha=\frac{1}{n} \sum_{i=1}^{n}(\frac{k}{2} \mathbf{d}_{i, k+1}-\frac{1}{2} \sum_{j=1}^{k} \mathbf{d}_{i j}).
\end{equation}

\subsection{Derivation of the Orthogonalization}
\label{orthgo}
Recalling Eq. (\ref{eq12}) in the main text:
\begin{equation}
    \mathbf{Y}=\sqrt{n} \mathbf{P}\left(\mathbf{R}^{-1}\right),
    \label{eq42}
\end{equation}
where $\mathbf{P}$ is the cluster assignment matrix, and $\mathbf{R}$ is a upper triangular matrix obtained from the QR decomposition $\mathbf{P} = \mathbf{E}\mathbf{R}$ and $\mathbf{E}^T\mathbf{E}=\mathbf{I}$.
Then we have 
\begin{equation}
    \mathbf{P}\left(\mathbf{R}^{-1}\right) =  \mathbf{E}.
\end{equation}
We further have
\begin{equation}
\begin{aligned}
    \mathbf{Y}^T\mathbf{Y} &= n \mathbf{E}^{T}\mathbf{E}\\
    &=n\mathbf{I}.
    \end{aligned}
\end{equation}
Therefore, we can obtain that $\mathbf{Y}$ is orthogonal.

\section{Experimental Settings}
\label{settings}
This section  provides detailed experimental settings in Section Experiments, including the description of all datasets in Section \ref{desriptiondata}, summarization of all comparison methods in Section \ref{descriptionmethods}, evaluation protocol in Section \ref{descriptionevaluation}, model architectures and settings in Section \ref{descriptionmodel}, and computing resource details in Section \ref{details}.

\begin{table}[!t]
\small
\centering
\setlength\tabcolsep{4pt}
\begin{threeparttable}[b]
\caption{Statistics of all datasets.}
\begin{tabular}{cccccccccc}
\toprule
Datasets &Type &\#Nodes  & \#Node Types & \#Edges & \#Edge Types &Target Node &\#Training & \#Test \\ \midrule 
ACM& Heter & 8,994 & 3 & 25,922 & 4 & Paper &600 &2,125\\
 \midrule
 Yelp& Heter & 3,913 & 4 & 72,132 & 6 & Bussiness &300 &2,014\\
 \midrule
  DBLP& Heter &  18,405 & 3 &  67,946 & 4 & Author &800 &2,857\\
 \midrule
   Aminer& Heter &  55,783 & 3 &  153,676 & 4 & Paper &80 &1,000\\
 \midrule
     Photo& Homo &  7,650 & 1 &  238,162 & 2 & Photo &765 &6,120\\ 
 \midrule
      Computers& Homo &  13,752 & 1 &  491,722 & 2 & Computer &1,375 &11,002\\ 
 \bottomrule
\end{tabular}
\label{dataset}
\end{threeparttable}
\end{table}

\subsection{Datasets}
\label{desriptiondata}
We use four public heterogeneous graph datasets and two public homogeneous graph datasets from various domains.  Heterogeneous graph datasets include three academic datasets (\ie ACM \cite{WangJSWYCY19}, DBLP \cite{WangJSWYCY19}, and Aminer \cite{kddHuFS19}), and one business dataset (\ie Yelp \cite{ZhaoWSHSY21}). Homogeneous graph datasets include two sale datasets (\ie Photo and Computers \cite{shchur2018pitfalls}). Table \ref{dataset} summarizes the data statistics. We list the details of the datasets as follows.

\begin{itemize}
\item \textbf{ACM}  is an academic heterogeneous graph dataset. It contains three types of nodes (paper (P), author (A), subject (S)), four types of edges (PA, AP, PS, SP), and categories of papers as labels.
\item \textbf{Yelp} is a business heterogeneous graph dataset. It contains four types of nodes (business (B), user (U), service (S), level (L)), six types of edges (BU, UB, BS, SB, BL, LB), and categories of businesses as labels.

\item \textbf{DBLP} is an academic heterogeneous graph dataset. It contains three types of nodes (paper (P), authors (A), conference (C)), four types of edges (PA, AP, PC, CP), and research areas of authors as labels.

\item \textbf{Aminer} is an academic heterogeneous graph dataset. It contains three types of nodes (paper (P), author (A), reference (R)), four types of edges (PA, AP, PR, RP), and categories of papers as labels.


\item \textbf{Photo} and \textbf{Computers} are two co-purchase homogeneous graph datasets. They are two networks extracted from Amazon’s co-purchase data. Nodes are products, and edges denote that these products were often bought together. Products are categorized into several classes by the product category.

\end{itemize}

\begin{table*}[ht]
\small
\centering
\setlength\tabcolsep{8pt}
\caption{The characteristics of all comparison methods.}
\begin{tabular}{rccccccccccc}
\hline
Methods &Hetero &Homo &Semi-sup  &Self-sup/unsup &Meta-path &Adaptive \\\hline
  DeepWalk (2014)& &  $\checkmark$&  &$\checkmark$ & \\\hline
  GCN (2017)&  &  $\checkmark$& $\checkmark$ & & \\ \hline
 GAT (2018)&  &  $\checkmark$& $\checkmark$ & & \\\hline
 DGI (2019)&  &  $\checkmark$&  &$\checkmark$ & \\\hline
 GMI (2020)&   &  $\checkmark$&  & $\checkmark$& \\\hline
 MVGRL (2020)&   &  $\checkmark$&  & $\checkmark$& \\\hline
 GRACE (2020)&  &  $\checkmark$&  &$\checkmark$ & \\\hline
 GCA (2021)&  &  $\checkmark$&  & $\checkmark$& \\\hline
 G-BT (2022)& &  $\checkmark$&  &$\checkmark$ & \\\hline
 COSTA (2022)&  &  $\checkmark$&  & $\checkmark$& \\\hline
 DSSL (2022)&  &  $\checkmark$&  & $\checkmark$&\\\hline
 LRD (2023)&  &  $\checkmark$&  & $\checkmark$& \\\hline
Mp2vec (2017)&$\checkmark$  &  &  &$\checkmark$ &$\checkmark$ \\\hline
 HAN (2019)&  $\checkmark$&  &$\checkmark$  & &$\checkmark$ \\\hline
   HGT (2020)&  $\checkmark$&  &$\checkmark$  & &\\\hline
 DMGI (2020)&$\checkmark$  &  &  & $\checkmark$&$\checkmark$ \\\hline
 DMGIattn (2020)&$\checkmark$  &  &  &$\checkmark$ &$\checkmark$ \\\hline
 HDMI (2021)&$\checkmark$  &  &  &$\checkmark$ &$\checkmark$ \\\hline
 HeCo (2021)&$\checkmark$  &  &  &$\checkmark$ &$\checkmark$ \\\hline
 HGCML (2023)&$\checkmark$  &  &  &$\checkmark$ &$\checkmark$ \\\hline
 CPIM (2023)&$\checkmark$  &  &  & $\checkmark$&$\checkmark$ \\\hline
  HGMAE (2023)&$\checkmark$  &  &  &$\checkmark$ &$\checkmark$ \\\hline
 HERO (2024) &$\checkmark$  &  &  &$\checkmark$&& $\checkmark$ \\ \hline
  SCHOOL (ours) &$\checkmark$  &  &  &$\checkmark$&& $\checkmark$ \\ \hline
\end{tabular}
\label{tabcomparison}
\end{table*}

\subsection{Comparison Methods}
\label{descriptionmethods}
The comparison methods include eleven heterogeneous graph methods and twelve homogeneous graph methods. Heterogeneous graph methods include Mp2vec \cite{metapath2veckddDongCS17}, HAN \cite{WangJSWYCY19}, HGT \cite{wwwHuDWS20}, DMGI \cite{DMGIParkK0Y20}, DMGIattn \cite{DMGIParkK0Y20}, HDMI \cite{jing2021hdmi}, HeCo \cite{WangLHS21}, HGCML \cite{wang2023heterogeneous}, CPIM \cite{CPIM}, HGMAE \cite{HGMAE}, and HERO \cite{mo2024selfsupervised}. Homogeneous graph methods include GCN \cite{kipf2017semisupervised}, GAT \cite{velickovic2018graph}, DeepWalk \cite{perozzi2014deepwalk}, DGI \cite{VelickovicFHLBH19}, GMI \cite{peng2020graph}, MVGRL \cite{hassani2020contrastive}, GRACE \cite{zhu2020deep}, GCA \cite{zhu2021graph}, G-BT \cite{GBT_KBS}, COSTA \cite{zhang2022costa}, DSSL \cite{xiaodecoupled}, and LRD \cite{yang2024self}. The characteristics of all methods are listed in Table \ref{tabcomparison}, where ``Hetero'' and ``Homo'' indicate the methods designed for the heterogeneous graph and homogeneous graph, respectively. ``Semi-sup",  and ``Self-sup/unsup" indicate that the method conducts semi-supervised learning, and self-supervised/unsupervised learning, respectively. ``Meta-path'' indicates that the method requires pre-defined meta-paths during the training process. ``Adaptive'' indicates that the method learns an adaptive graph structure instead of traditional meta-paths.

\subsection{Evaluation Protocol}
\label{descriptionevaluation}
We follow the evaluation in previous works \cite{jing2021hdmi, pan2021multi,CKD00010YCLF022}  to conduct node classification and node clustering
as semi-supervised and unsupervised downstream tasks, respectively. Specifically, we first train models with unlabeled data in a self-supervised manner and output learned node representations. After that, the resulting representations can be used for different downstream tasks. For the node classification task, we train a simple logistic regression classifier with a fixed iteration number, and then evaluate the effectiveness of all methods with Micro-F1 and Macro-F1 scores. For the node clustering task, we conduct clustering and split the learned representations into $c$ clusters with the K-means algorithm, then calculate the normalized mutual information (NMI) and average rand index (ARI) to evaluate the performance of node clustering.

\begin{table*}[t]
\centering
\setlength\tabcolsep{14pt}
\caption{ Settings for the dimensions of encoders (i.e., $g_{\phi}\in\mathbb{R}^{f\times d_1}$ and $f_\theta\in\mathbb{R}^{f\times d_1}$) and projection heads (i.e., $p_{\varphi}\in\mathbb{R}^{d_1\times c}$ and  $q_\gamma\in\mathbb{R}^{d_1\times d_2}$) on all datasets.}
\begin{tabular}{lcccccccccl}
\toprule
 Settings& ACM  & Yelp & DBLP & Aminer & Photo & Computers\\ \midrule
$f$ & 1,902 & 82 & 334 & 128 & 745 & 767 \\
$d_1$ & 512 & 256  & 128 & 256 & 1024 & 1024 \\
$d_2$ & 64 & 256 & 256 & 256 & 256 & 256 \\
$c$ & 3 & 3 & 4 & 4 & 8 & 10 \\
 \bottomrule
\end{tabular}
\label{tabsettings}
\end{table*}

\begin{table*}[t]
\small
\setlength{\tabcolsep}{4.5pt}
\centering
\caption{Clustering performance (\ie NMI and ARI) of all methods on  heterogeneous graph datasets.}
\begin{tabular}{lccccccccccccc}
\toprule
\multirow{2}{*}{\textbf{Method}}&\multicolumn{2}{c}{\textbf{ACM}}& \multicolumn{2}{c}{\textbf{Yelp}}& \multicolumn{2}{c}{\textbf{DBLP}}& \multicolumn{2}{c}{\textbf{Aminer}} \\
\cmidrule(r){2-3} \cmidrule(r){4-5} \cmidrule(r){6-7} \cmidrule(r){8-9}
&NMI&ARI&NMI&ARI &NMI&ARI &NMI&ARI\\
\midrule
DeepWalk& 41.6$\pm$0.5 & 35.3$\pm$0.6  & 35.1$\pm$0.8 &37.7$\pm$1.1 & 69.0$\pm$0.2  & 73.3$\pm$0.3 & 26.2$\pm$0.3&22.4$\pm$0.4 \\
Mp2vec& 21.4$\pm$0.7 & 21.1$\pm$0.5 &  38.9$\pm$0.6 & 39.5$\pm$0.5 &  73.5$\pm$0.4 & 77.7$\pm$0.6 & 30.4$\pm$0.4&25.5$\pm$0.6\\
\midrule
DMGI &67.8$\pm$0.9 & 70.2$\pm$1.0 & 36.8$\pm$0.6 &34.4$\pm$0.7  &72.2$\pm$0.8 &72.8$\pm$0.9&27.3$\pm$0.9&23.1$\pm$0.8\\
DMGIattn&\textbf{70.2$\pm$0.3} & 72.5$\pm$0.6  & 38.1$\pm$0.8 &40.2$\pm$0.6 & 69.6$\pm$0.6 &73.9$\pm$0.4 & 28.3$\pm$0.3&25.5$\pm$0.5\\
HDMI&69.5$\pm$0.5 & 72.3$\pm$0.7 & 38.9$\pm$0.6 &40.7$\pm$0.8  &73.1$\pm$0.3 &74.4$\pm$0.4 &33.5$\pm$0.4&28.9$\pm$0.5\\
HeCo& 67.8$\pm$0.8 & 70.5$\pm$0.7 & 39.3$\pm$0.6 &42.1$\pm$0.8  &74.5$\pm$0.8 &80.1$\pm$0.9 &32.2$\pm$1.1&28.6$\pm$1.0\\
HGCML&69.1$\pm$0.7 & 71.6$\pm$0.8 & 37.4$\pm$0.6 &39.5$\pm$0.8  &74.5$\pm$0.9 &75.1$\pm$1.1 &35.9$\pm$0.6&31.1$\pm$0.5\\
CPIM&68.6$\pm$0.3 &70.8$\pm$0.5 & 40.1$\pm$0.8 &42.1$\pm$0.9 & 73.7$\pm$0.5 &78.0$\pm$0.3 & 35.8$\pm$0.5 &30.1$\pm$0.7\\
HGMAE&69.7$\pm$0.8 & 72.6$\pm$0.6 & 40.3$\pm$0.9 &42.4$\pm$0.8  &76.9$\pm$0.6 &82.3$\pm$0.7 &41.1$\pm$0.8 &38.3$\pm$0.9\\
HERO&68.8$\pm$0.6 &71.8$\pm$0.6  & 38.6$\pm$0.8&40.6$\pm$0.9 & 74.1$\pm$0.7 &79.3$\pm$0.7 & 36.8$\pm$0.7 &35.3$\pm$0.9\\
\textbf{SCHOOL}& 69.6$\pm$0.7 &\textbf{72.7$\pm$0.5}  & \textbf{41.2$\pm$0.9} &\textbf{43.5$\pm$0.6} & \textbf{77.1$\pm$0.6} &\textbf{82.5$\pm$0.5} & \textbf{42.4$\pm$0.6} &\textbf{38.8$\pm$0.8}\\
\bottomrule
\end{tabular}
\label{tabclu}
\end{table*}

\subsection{Model Architectures and Settings}
\label{descriptionmodel}
As described in Section \ref{method}, the proposed method employs the MLP (\ie $g_\phi$) and the closed-form solution of the affinity matrix $\mathbf{S}$ to obtain node representations $\mathbf{Z}$. Moreover, the proposed method employs the heterogeneous encoder (\ie $f_\theta$) to obtain heterogeneous representations $\widetilde{\mathbf{Z}}$. In addition, the proposed method employs the projection head $p_\varphi$ to obtain the cluster assignment matrix $\mathbf{P}$. After that, the proposed method employs projection head $q_\gamma$ to map the node representations and heterogeneous representations into latent spaces. In the proposed method, projection head $p_\varphi$ and $q_\gamma$  are simply implemented by the linear layer, followed by the ReLU activation. We report the settings for the dimensions of encoders in Table \ref{tabsettings}.  Finally, In the proposed method,  all parameters were optimized by the Adam optimizer \cite{kingma2015adam} with an initial learning rate. Moreover, We use early
stopping with a patience of 30 to train the proposed SHGL model. In all experiments, we repeat the experiments five times for all methods and report the average results. 


\subsection{Computing Resource Details}
\label{details}
All experiments were implemented in PyTorch and conducted on a server with 8 NVIDIA GeForce 3090 (24GB memory each). Almost every experiment can be done on an individual 3090, and the training time of all comparison methods as well as our method, is less than 1 hour.

\section{Additional Experiments}
\label{additional_results}
This section provides some additional experimental results to support the proposed method, including experiments on the effectiveness of the affinity matrix in Section \ref{effectiveness_affinity}, visualization of the learned representations in Section \ref{vis_node}, parameter analysis in Section \ref{PA}, experimental results on the node clustering task in Table \ref{tabclu}, and experimental results on homogeneous graph datasets in Table \ref{singleviewdatasets}.

\begin{table*}[t]
\footnotesize
\centering
\caption{Classification performance (\ie Macro-F1 and Micro-F1) on  homogeneous graph datasets.}
\begin{tabular}{lccccccccccccc}
\toprule
\multirow{2}{*}{\textbf{Method}}&\multicolumn{2}{c}{\textbf{Photo}}& \multicolumn{2}{c}{\textbf{Computers}}&  \\
\cmidrule(r){2-3} \cmidrule(r){4-5} 
&Macro-F1&Micro-F1&Macro-F1&Micro-F1\\
\midrule
DeepWalk& 87.4$\pm$0.5 &89.7$\pm$0.3  &84.0$\pm$0.3 &85.6$\pm$0.4 \\
GCN & 90.5$\pm$0.3 &92.5$\pm$0.2  &84.0$\pm$0.4 &86.4$\pm$0.3  \\
GAT & 90.2$\pm$0.5 &91.8$\pm$0.4  &83.2$\pm$0.2 &85.7$\pm$0.4 \\
\midrule
DGI&89.3$\pm$0.2 & 91.6$\pm$0.3  &79.3$\pm$0.3 &83.9$\pm$0.5 \\
GMI& 89.3$\pm$0.4 &90.6$\pm$0.2  &80.1$\pm$0.4 &82.2$\pm$0.4 \\
MVGRL& 90.1$\pm$0.3 &91.7$\pm$0.4  &84.6$\pm$0.6 &86.9$\pm$0.5 
\\
GRACE& 90.3$\pm$0.5 &91.9$\pm$0.3  &84.2$\pm$0.3 &86.8$\pm$0.5 
\\
GCA& 91.1$\pm$0.4 &92.4$\pm$0.4 &85.9$\pm$0.5 &87.7$\pm$0.3 
\\
COSTA&91.3$\pm$0.4 &92.5$\pm$0.3 &86.4$\pm$0.3 &88.3$\pm$0.4
\\
DSSL&90.6$\pm$0.2 &92.1$\pm$0.3  &85.6$\pm$0.3 &87.3$\pm$0.4 
\\
LRD&91.1$\pm$0.5 &92.8$\pm$0.7  &\textbf{86.6$\pm$0.3} &88.6$\pm$0.6 
\\
\textbf{SCHOOL}& \textbf{91.9$\pm$0.4} &\textbf{93.1$\pm$0.3}  & 85.9$\pm$0.6 &\textbf{88.7$\pm$0.5}\\
\bottomrule
\end{tabular}
\label{singleviewdatasets}
\end{table*}

\begin{figure*}[t]
\centering
        \subfigure[HDMI (SIL: 0.36)]{\scalebox{0.2}{\includegraphics{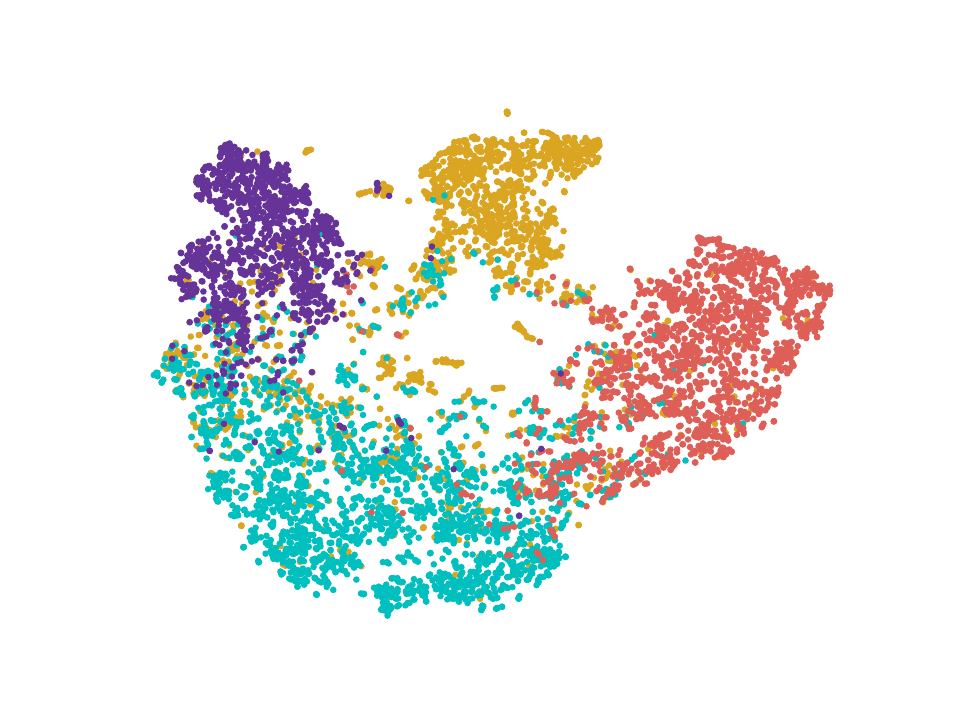}  
        }\label{hdmi}}
        \subfigure[HeCo (SIL: 0.33)]{\scalebox{0.2}
{\includegraphics{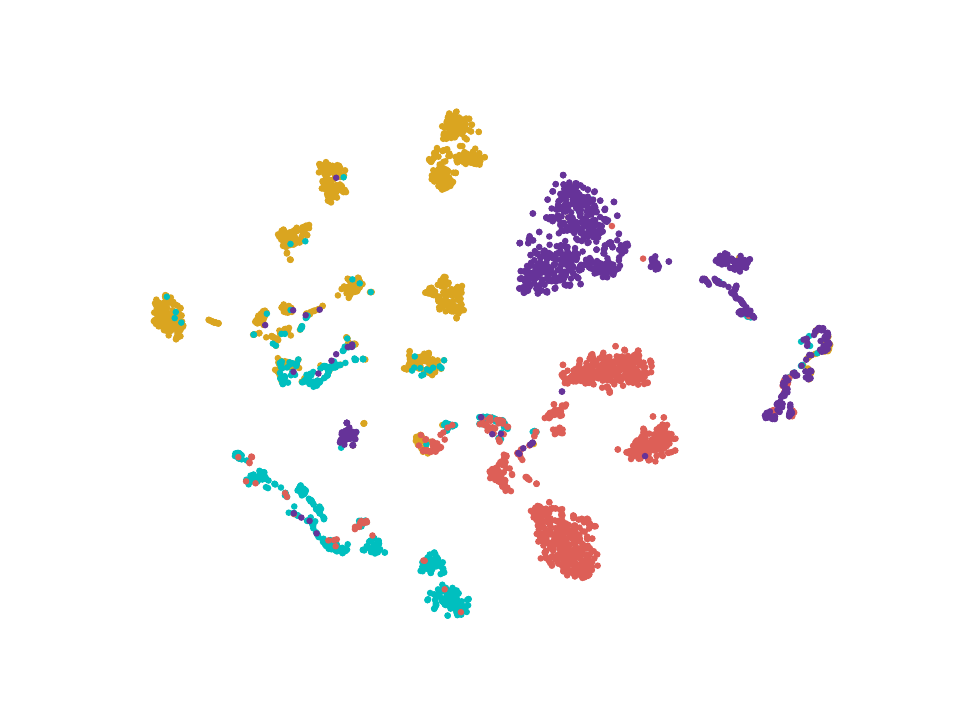}
        }\label{heco}}
        \subfigure[HERO (SIL: 0.34)]{\scalebox{0.2}
        {\includegraphics{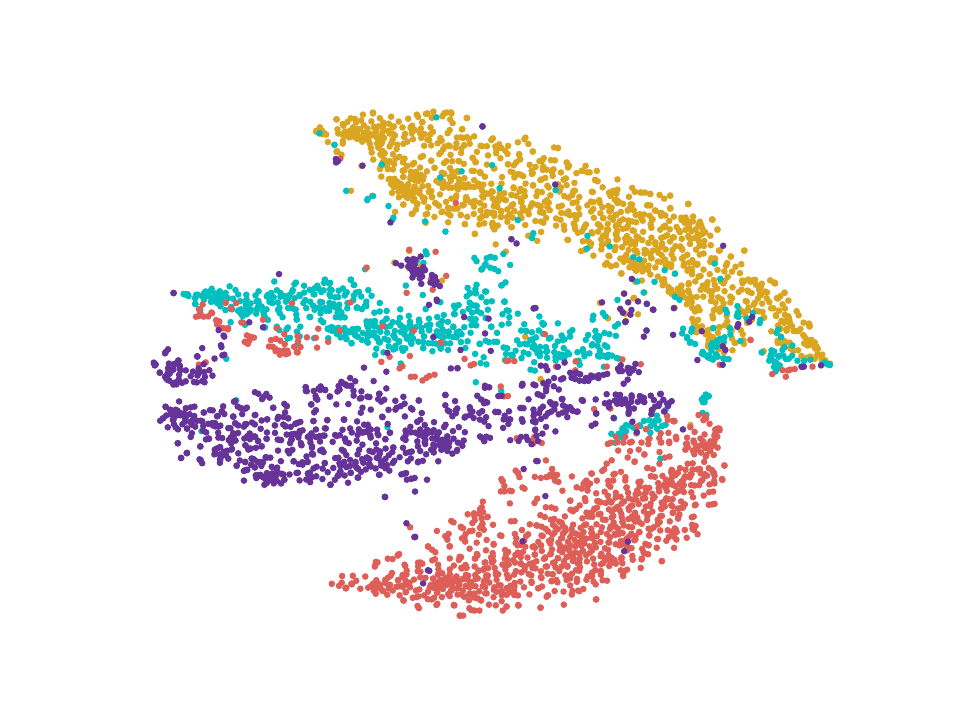}
        }\label{hero}}
        \subfigure[SCHOOL (SIL: 0.41)]{\scalebox{0.2}
{\includegraphics{DBLP_tsne_SIL_0.4502212405204773.pdf}
        }\label{schgl}}
		\caption{Visualization plotted by t-SNE and the corresponding silhouette scores (SIL) of node representations of the proposed SCHOOL and other SHGL comparison methods on the DBLP dataset. }
 \label{vistsne}
\end{figure*}

\subsection{Effectiveness of the Rank-Constrained Affinity Matrix}
\label{effectiveness_affinity}
The proposed method proposes to learn a rank-constrained affinity matrix with exact $c$ components to capture the connections within the same class while mitigating the connections from different classes. This actually shares part of a similar idea with the self-attention mechanism, which aims to assign weights for all sample pairs. To further verify the effectiveness of the rank-constrained affinity matrix, we investigate the performance of the variants methods with the cosine similarity, the affinity matrix, the self-attention mechanism, and report the results in Table \ref{tab_attentionandexpressive}.

Obviously, the proposed method with the affinity matrix obtains superior performance than the cosine similarity and the self-attention mechanism on all datasets. The reason can be attributed to the fact that the affinity matrix in the proposed method is constrained to contain exactly $c$ components to mitigate noisy connections from different classes. In contrast, although either the cosine similarity or self-attention mechanisms may assign small weights for node pairs from different classes, it inevitably introduces noise during the message-passing process to affect the quality of node representations. As a result, the effectiveness of the rank-constrained affinity matrix is verified.

\subsection{Effectiveness of the Node-level Consistency Constraint}
To verify the effectiveness of the node-level consistency constraint, we conducted experiments to replace the proposed node-level consistency constraint with the InfoNCE loss and reported the results in Table \ref{tab_atten}. From Table \ref{tab_atten}, we can find that the variant method with InfoNCE loss obtains a similar performance to the proposed method. However, the InfoNCE loss generally requires the time complexity of $\mathcal{O}(n^2)$, where $n$ is the number of nodes. This may introduce large computation costs during the training process. In contrast, the proposed method simply designs the node-level consistency constraint in Eq. (15) to capture the invariant information with the time complexity of $\mathcal{O}(nd^2)$, where $d$ is the representation dimension and generally $d^2<n$. 

\begin{table*}[t]
\small
\centering
\setlength\tabcolsep{3pt}
\caption{Classification performance (\ie Macro-F1 and Micro-F1) of variant methods with the affinity matrix, cosine similarity and, self-attention mechanisms on heterogeneous graph datasets.}
\begin{tabular}{lccccccccccccc}
\toprule
\multirow{2}{*}{\textbf{Method}}&\multicolumn{2}{c}{\textbf{ACM}}& \multicolumn{2}{c}{\textbf{Yelp}}& \multicolumn{2}{c}{\textbf{DBLP}}& \multicolumn{2}{c}{\textbf{Aminer}} \\
\cmidrule(r){2-3} \cmidrule(r){4-5} \cmidrule(r){6-7} \cmidrule(r){8-9}
&Macro-F1&Micro-F1&Macro-F1&Micro-F1 &Macro-F1&Micro-F1 &Macro-F1&Micro-F1\\
\midrule
cosine similarity& 85.3$\pm$0.9 &85.1$\pm$1.1  &88.2$\pm$0.4 &87.7$\pm$0.7 &89.3$\pm$0.7 &90.5$\pm$0.8&67.6$\pm$0.5 &75.4$\pm$0.6\\
self-attention& 88.7$\pm$0.8 & 88.4$\pm$0.7 & 92.0$\pm$0.5 &91.7$\pm$0.6  &91.2$\pm$0.4 &92.1$\pm$0.6 &73.2$\pm$0.7&82.1$\pm$0.6\\
affinity matrix& \textbf{92.7$\pm$0.6} & \textbf{92.6$\pm$0.5} & \textbf{93.0$\pm$0.7} &\textbf{92.8$\pm$0.4}  &\textbf{94.0$\pm$0.3} &\textbf{94.7$\pm$0.4} &\textbf{77.5$\pm$0.9} &\textbf{86.8$\pm$0.7}\\
\bottomrule
\end{tabular}
\label{tab_attentionandexpressive}
\end{table*}

\begin{table*}[t]
\small
\centering
\setlength\tabcolsep{4.4pt}
\caption{Classification performance (\ie Macro-F1 and Micro-F1) of the affinity matrix and self-attention mechanisms on heterogeneous graph datasets.}
\begin{tabular}{lccccccccccccc}
\toprule
\multirow{2}{*}{\textbf{Method}}&\multicolumn{2}{c}{\textbf{ACM}}& \multicolumn{2}{c}{\textbf{Yelp}}& \multicolumn{2}{c}{\textbf{DBLP}}& \multicolumn{2}{c}{\textbf{Aminer}} \\
\cmidrule(r){2-3} \cmidrule(r){4-5} \cmidrule(r){6-7} \cmidrule(r){8-9}
&Macro-F1&Micro-F1&Macro-F1&Micro-F1 &Macro-F1&Micro-F1 &Macro-F1&Micro-F1\\
\midrule
InfoNCE& 91.3$\pm$1.1 & 91.2$\pm$0.8 & 92.4$\pm$0.7 &92.0$\pm$0.8  &94.1$\pm$0.6 &94.6$\pm$0.5 &76.8$\pm$0.4&85.4$\pm$0.3\\
Proposed& \textbf{92.7$\pm$0.6} & \textbf{92.6$\pm$0.5} & \textbf{93.0$\pm$0.7} &\textbf{92.8$\pm$0.4}  &\textbf{94.0$\pm$0.3} &\textbf{94.7$\pm$0.4} &\textbf{77.5$\pm$0.9} &\textbf{86.8$\pm$0.7}\\
\bottomrule
\end{tabular}
\label{tab_atten}
\end{table*}

\begin{figure*}[ht]
\centering
		{\includegraphics[scale=0.4]{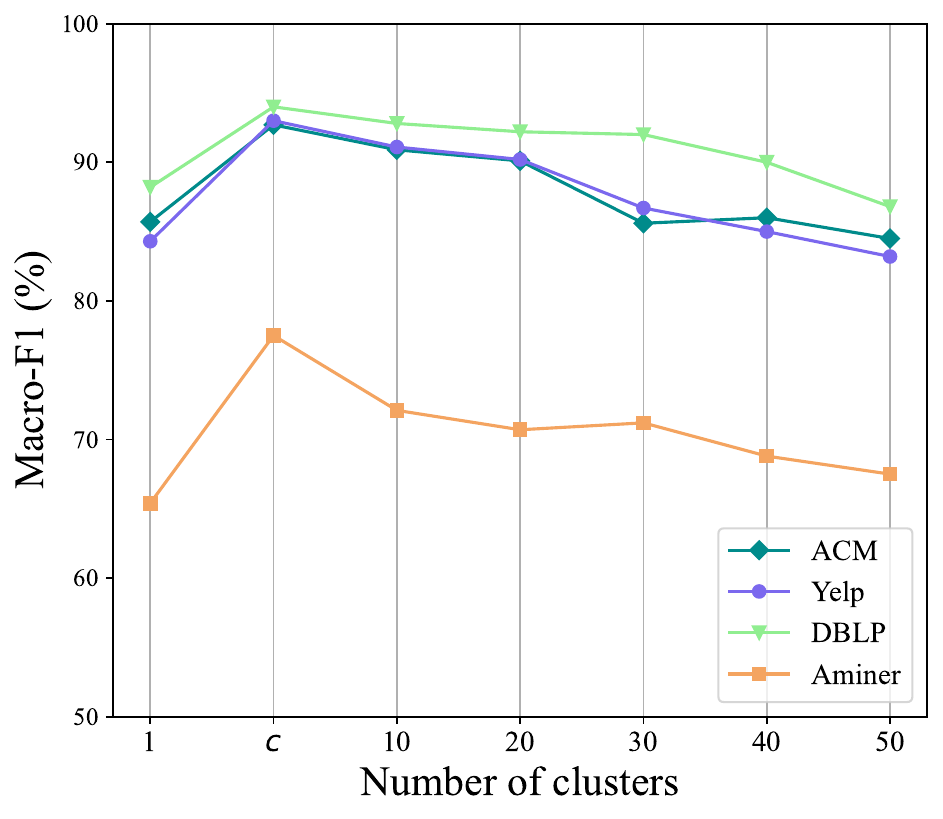}}
		\caption{Classification performance (i.e., Macro-F1) of the proposed method under different clusters.
  }
		\label{figframe}
\end{figure*}

\subsection{Effectiveness of Different Cluster Numbers}
The proposed method divides the learned representations into several clusters. Generally, the number of clusters equals to $c$ obtains better results because, in downstream tasks, it is easier to distinguish $c$ clusters than a larger number of clusters. To verify it, we changed the number of clusters and reported the results in Figure \ref{figframe}. Obviously, the proposed method obtains the best results when the number of clusters equals to $c$ and decreases as the number of classes increases. This is reasonable because when the number of classes increases, the nodes within the same class may be assigned to different clusters, thus making it difficult to classify them correctly.

\subsection{Visualization of the Learned Representations}
\label{vis_node}
To further verify the effectiveness of the learned representations, we visualize node representations of the proposed SCHOOL and other SHGL comparison methods on the DBLP dataset and report the results and corresponding silhouette scores (SIL) in Figure \ref{vistsne}. Obviously, in the visualization, the node representations learned by the proposed method exhibit better clustering status, \ie nodes with different class labels are more widely separated. Moreover,  the representations learned by the proposed method obtain the best silhouette score, compared to other SHGL comparison methods (\ie HDMI, HeCo, and HERO). The reason can be attributed to the fact that the proposed method conducts spectral clustering explicitly, and cuts the learned graph into $c$ components as well as further utilizes the clustering information to facilitate downstream tasks.

\subsection{Parameter Analysis}
\label{PA}
In the proposed method, we employ non-negative parameters (\ie $\mu$, and $\delta$) to achieve a trade-off between different terms of the final objective function $\mathcal{J}$. To investigate the impact of $\mu$, and $\delta$ with different settings, we conduct the node classification on all heterogeneous graph datasets by varying the value of parameters in the range of $[10^{-3}$,$10^{3}]$ and reporting the results in Figure \ref{parafig}. From Figure \ref{parafig}, we can find that if the values of parameters are too small (\eg $10^{-3}$), the proposed method cannot achieve satisfactory performance. This verifies that both node-level and cluster-level consistency constraints are significant for the proposed method.


\begin{figure*}[t]
\centering
        \subfigure[ACM]{\scalebox{0.195}{\includegraphics{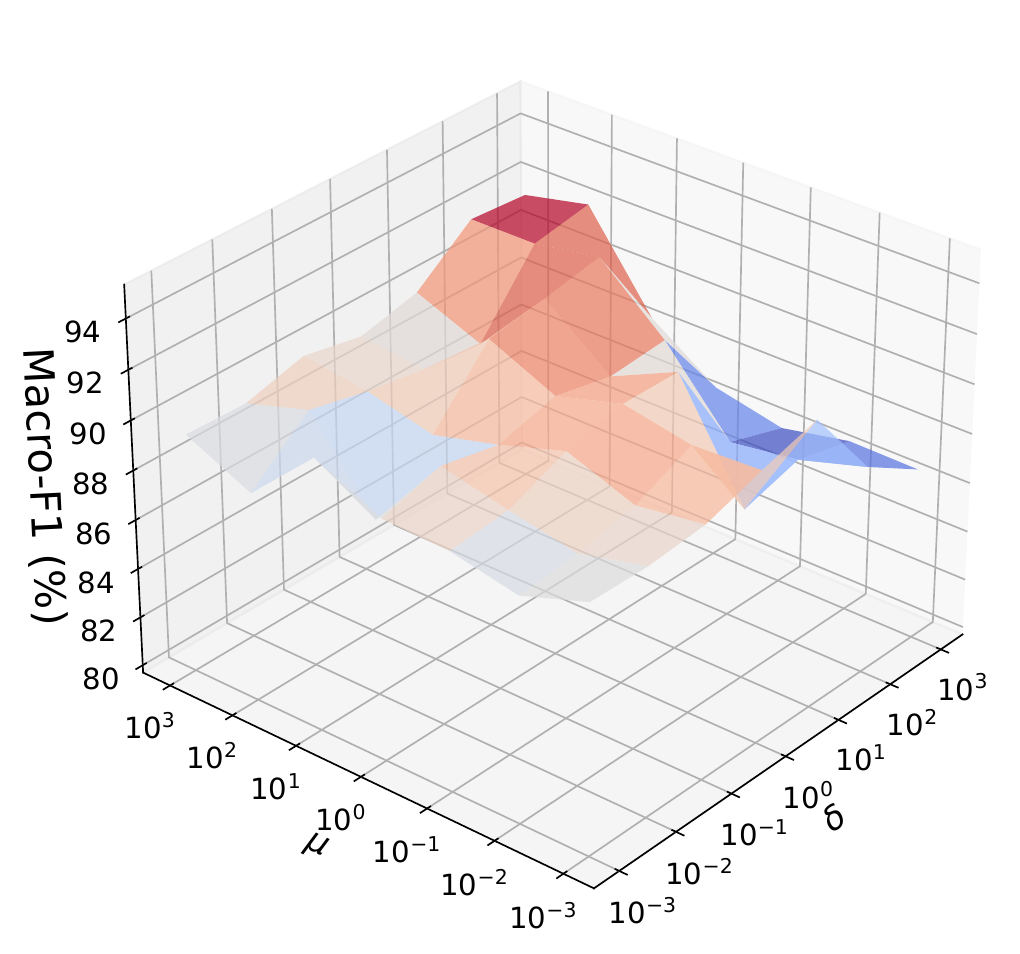}  
        }\label{para_1a}}
        \subfigure[Yelp]{\scalebox{0.195}
{\includegraphics{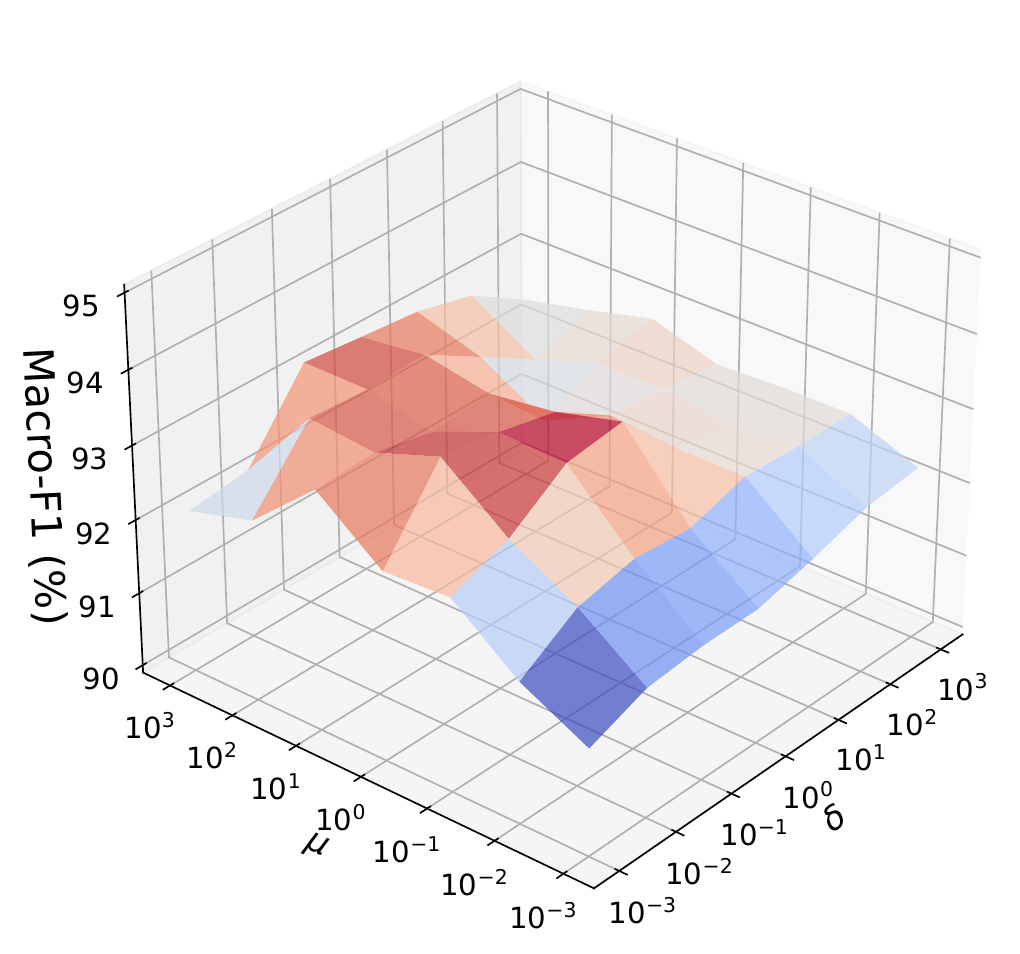}
        }\label{para_1b}}
        \subfigure[DBLP]{\scalebox{0.195}
        {\includegraphics{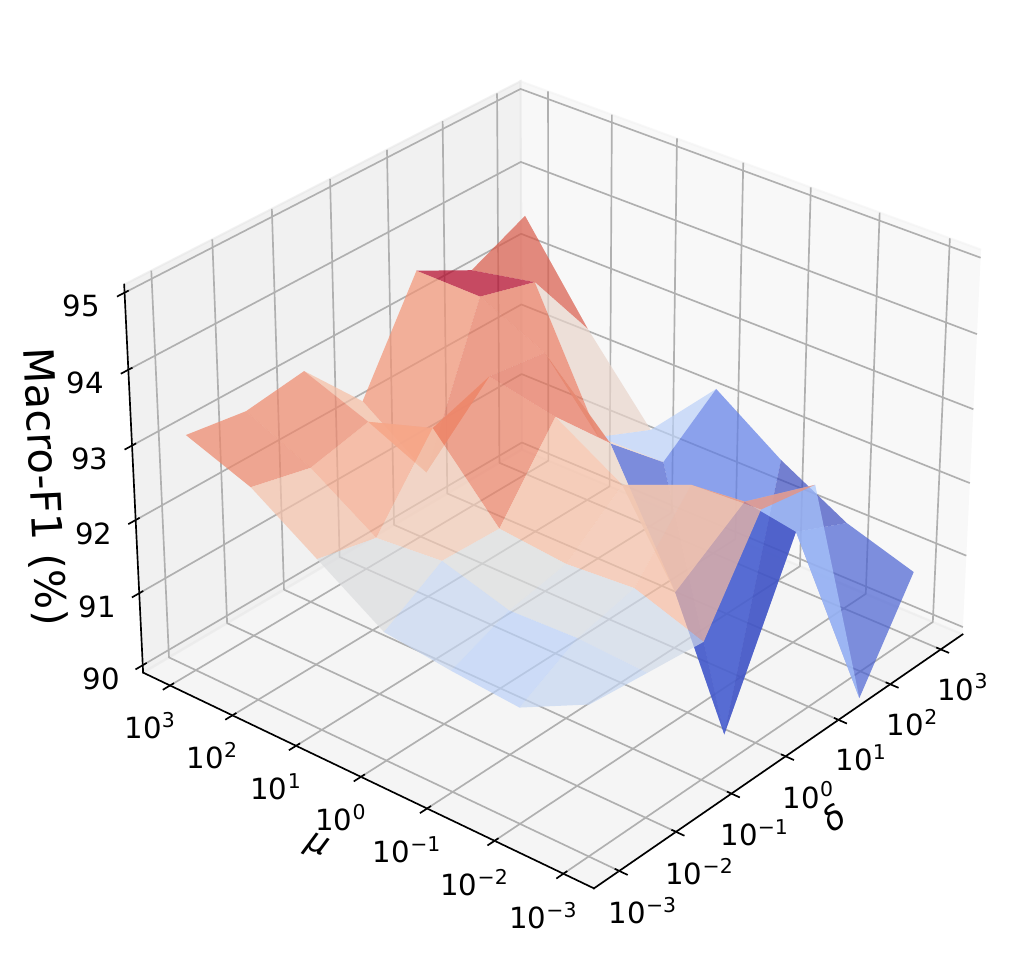}
        }\label{para_1c}}
        \subfigure[Aminer]{\scalebox{0.195}
{\includegraphics{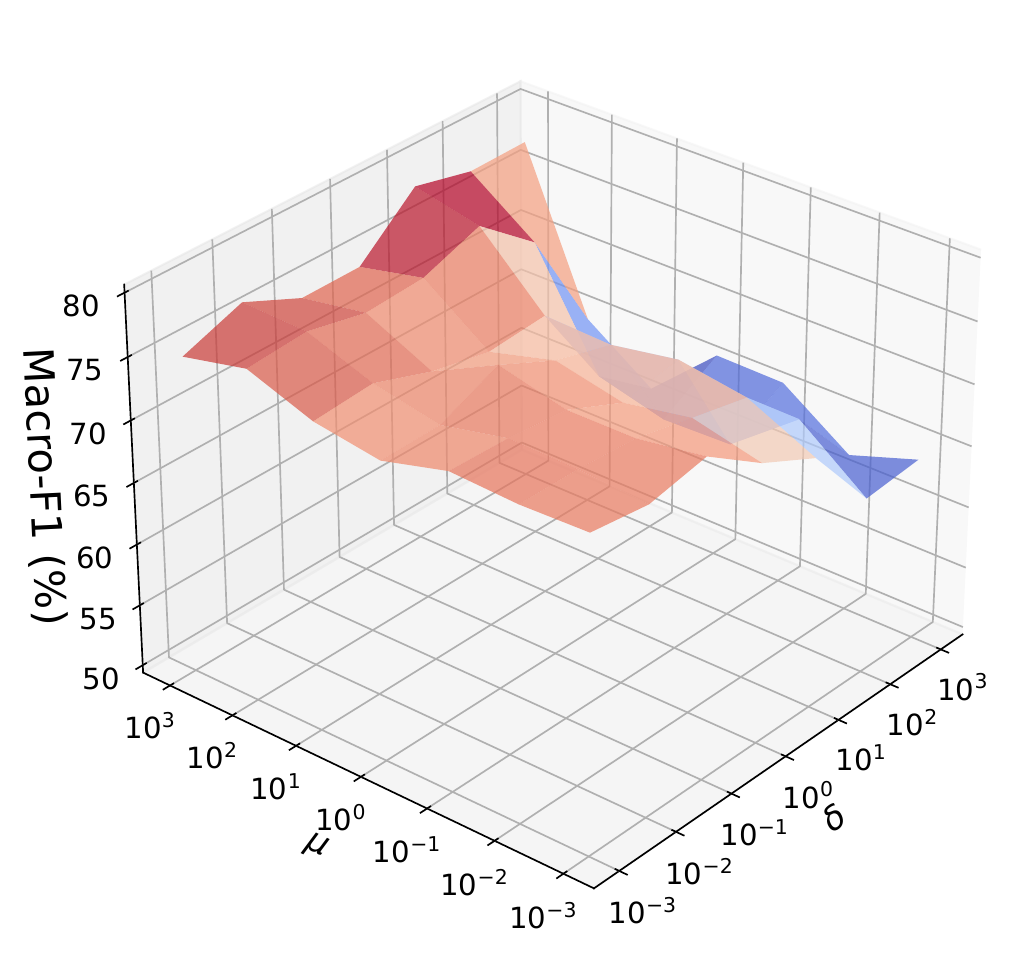}
        }\label{para_1d}}
		\caption{The classification performance of the proposed method at different parameter settings (\ie $\mu$, and $\delta$) on all heterogeneous graph datasets. }
 \label{parafig}
\end{figure*}

\clearpage
\section*{NeurIPS Paper Checklist}
\begin{enumerate}

\item {\bf Claims}
    \item[] Question: Do the main claims made in the abstract and introduction accurately reflect the paper's contributions and scope? 
    \item[] Answer: \answerYes{} 

\item {\bf Limitations}
    \item[] Question: Does the paper discuss the limitations of the work performed by the authors?
    \item[] Answer: \answerYes{} We discuss the limitations of the work in Section \ref{limitations}.

\item {\bf Theory Assumptions and Proofs}
    \item[] Question: For each theoretical result, does the paper provide the full set of assumptions and a complete (and correct) proof?
    \item[] Answer: \answerYes{} We provide the assumptions and complete proof in Appendix \ref{proofs}. 

    \item {\bf Experimental Result Reproducibility}
    \item[] Question: Does the paper fully disclose all the information needed to reproduce the main experimental results of the paper to the extent that it affects the main claims and/or conclusions of the paper (regardless of whether the code and data are provided or not)?
    \item[] Answer: \answerYes{} See Section \ref{experiments} and Appendix \ref{settings}.

\item {\bf Open access to data and code}
    \item[] Question: Does the paper provide open access to the data and code, with sufficient instructions to faithfully reproduce the main experimental results, as described in supplemental material?
    \item[] Answer: \answerYes{} We released codes and data at \url{https://github.com/YujieMo/SCHOOL}.

\item {\bf Experimental Setting/Details}
    \item[] Question: Does the paper specify all the training and test details (e.g., data splits, hyperparameters, how they were chosen, type of optimizer, etc.) necessary to understand the results?
    \item[] Answer: \answerYes{} We specify all the training and test details in Appendix \ref{settings}. 

\item {\bf Experiment Statistical Significance}
    \item[] Question: Does the paper report error bars suitably and correctly defined or other appropriate information about the statistical significance of the experiments?
    \item[] Answer: \answerYes{} 

\item {\bf Experiments Compute Resources}
    \item[] Question: For each experiment, does the paper provide sufficient information on the computer resources (type of compute workers, memory, time of execution) needed to reproduce the experiments?
    \item[] Answer: \answerYes{} We list the details of experiments compute resources in Appendix \ref{details}.
    
\item {\bf Code Of Ethics}
    \item[] Question: Does the research conducted in the paper conform, in every respect, with the NeurIPS Code of Ethics \url{https://neurips.cc/public/EthicsGuidelines}?
    \item[] Answer: \answerYes{} 

\item {\bf Broader Impacts}
    \item[] Question: Does the paper discuss both potential positive societal impacts and negative societal impacts of the work performed?
    \item[] Answer: \answerYes{}  We discuss broder impacts in Section \ref{limitations}.

\item {\bf Safeguards}
    \item[] Question: Does the paper describe safeguards that have been put in place for responsible release of data or models that have a high risk for misuse (e.g., pretrained language models, image generators, or scraped datasets)?
    \item[] Answer: \answerNA{} 

\item {\bf Licenses for existing assets}
    \item[] Question: Are the creators or original owners of assets (e.g., code, data, models), used in the paper, properly credited and are the license and terms of use explicitly mentioned and properly respected?
    \item[] Answer: \answerYes{} 

\item {\bf New Assets}
    \item[] Question: Are new assets introduced in the paper well documented and is the documentation provided alongside the assets?
    \item[] Answer: \answerYes{} 

\item {\bf Crowdsourcing and Research with Human Subjects}
    \item[] Question: For crowdsourcing experiments and research with human subjects, does the paper include the full text of instructions given to participants and screenshots, if applicable, as well as details about compensation (if any)? 
    \item[] Answer: \answerNA{} 

\item {\bf Institutional Review Board (IRB) Approvals or Equivalent for Research with Human Subjects}
    \item[] Question: Does the paper describe potential risks incurred by study participants, whether such risks were disclosed to the subjects, and whether Institutional Review Board (IRB) approvals (or an equivalent approval/review based on the requirements of your country or institution) were obtained?
    \item[] Answer: \answerNA{} 
\end{enumerate}

\end{document}